%% file: iclr2022_conference.tex
\documentclass{article} %
\usepackage{iclr2022_conference,times}

\input{math_commands.tex}

\usepackage{hyperref}
\usepackage{url}

\usepackage{booktabs}       %
\usepackage{amsfonts}       %
\usepackage{nicefrac}       %
\usepackage{microtype}      %
\usepackage{xcolor}         %
\usepackage{amsmath}
\usepackage{multirow}
\usepackage{graphicx}
\usepackage{amsthm}
\usepackage{comment}
\usepackage{subcaption}
\usepackage{xspace}
\usepackage{enumitem}

\usepackage{mathtools}
\usepackage{chngcntr}

\newcommand{\method}{\textsc{GNP}\xspace}
\newcommand{\red}[1]{\textcolor{red}{#1}}
\newcommand{\blue}[1]{\textcolor{blue}{#1}}
\newcommand{\brown}[1]{\textcolor{brown}{#1}}

\theoremstyle{definition}
\newtheorem{proposition}{Proposition}
\newtheorem{theorem}{Theorem}

\newtheorem{lemma}{Lemma}

\counterwithin*{theorem}{section} 

\title{Learning to Pool in Graph Neural Networks for Extrapolation}

\author{Jihoon Ko, Taehyung Kwon, Kijung Shin, Juho Lee\thanks{Corresponding Author.} \\
Kim Jaechul Graduate School of AI\\
KAIST\\
Republic of Korea \\
\texttt{\{jihoonko,taehyung.kwon,kijungs,juholee\}@kaist.ac.kr} \\
}

\begin{document}

\maketitle

\begin{abstract}
\input{000abstract}

\end{abstract}

\section{Introduction}
\label{sec:intro}

\input{010introduction}

\section{Related work}
\label{sec:related}

\input{020related}

\section{Main Contribution: Generalized Norm-based Pooling}
\label{sec:method}
\input{040method}

\section{Experiments}
\label{sec:exp}
\input{050experiments}

\section{Conclusion}
\label{sec:conclusion}

\input{060conclusion}

\subsubsection*{Reproducibility Statement}

We provided the source code used in our experiments in main paper, including the implementations of \method and the GIN model,
in the supplementary materials. The provided supplementary matarials also include example synthetic datasets and the pretrained weights used in our experiments.

\bibliography{iclr2022_conference}
\bibliographystyle{iclr2022_conference}

\appendix
\input{supple}

\end{document}

%% file: math_commands.tex
\usepackage{amsmath,amsfonts,bm}

\def\eqref#1{equation~\ref{#1}}

\def\floor#1{\lfloor #1 \rfloor}
\def\1{\bm{1}}

\DeclareMathAlphabet{\mathsfit}{\encodingdefault}{\sfdefault}{m}{sl}
\SetMathAlphabet{\mathsfit}{bold}{\encodingdefault}{\sfdefault}{bx}{n}

%% file: 000abstract.tex
Graph neural networks (GNNs) are one of the most popular approaches to using deep learning on graph-structured data, and they have shown state-of-the-art performances on a variety of tasks. However, according to a recent study, a careful choice of pooling functions, which are used for the aggregation and readout operations in GNNs, is crucial for enabling GNNs to extrapolate. Without proper choices of pooling functions, which varies across tasks, GNNs completely fail to generalize to out-of-distribution data, while the number of possible choices grows exponentially with the number of layers. In this paper, we present GNP, a $L^p$ norm-like pooling function that is trainable end-to-end for any given task. Notably, GNP generalizes most of the widely-used pooling functions. We verify experimentally that simply using GNP for every aggregation and readout operation enables GNNs to extrapolate well on many node-level, graph-level, and set-related tasks; and GNP sometimes performs even better than the best-performing choices among existing pooling functions.

%% file: 010introduction.tex
Many real-world data, such as relationships between people in social networks or chemical bonds between atoms, can naturally be represented as graphs. Finding models with proper inductive biases to better describe such graph data has been a common goal for many researchers, and Graph Neural Networks (GNNs)~\citep{scarselli2009graph,kipf2017semi,hamilton2017inductive,velivckovic2017graph,xu2019powerful,maron2019provably,Xu2020Inductive} are considered to be the most successful model. They have proved effective for a variety of tasks, including recommendation \citep{ying2018graph}, drug discovery \citep{stokes2020deep}, and chip design \citep{mirhoseini2020chip}.

An important design choice for a GNN often overlooked is the specification of \emph{pooling functions}, the functions used for the aggregation or readout operation in GNNs.
They are usually required to be invariant w.r.t. the permutation of nodes in a graph, and common choices are element-wise summation (\texttt{sum}), maximum (\texttt{max}), minimum (\texttt{min}), or average (\texttt{mean}). Some recent works also proposed to use parametric models and learn them from data as well \citep{ying2018hierarchical,lee2019self,gao2019graph,yuan2020structpool}.

While most of the previous works on this line focused on improving predictive performance for their own tasks, recently, \citet{xu2021neural} studied the impact of the choice of pooling functions on the ability of a neural network to \emph{extrapolate}. Specifically, \citet{xu2021neural} highlighted the importance of the choice of pooling functions in order to make GNNs generalize over the data lying outside of the support of the training data distribution, and they argued that the specification of the pooling functions acts as an important inductive bias that can make GNNs either completely fail to extrapolate or gracefully generalize to out-of-distribution data. As a motivating example, consider the problem of counting the number of nodes in a graph. If we are to solve this problem with a single-layer GNN having one readout layer, probably the best pooling function would be \texttt{sum}, and the corresponding model will readily generalize to graphs with a much larger number of nodes than the ones seen during training. On the other hand, if we choose the pooling function as \texttt{max} instead, it may still fit the training data well but completely fail to predict the number of nodes in out-of-distribution graphs.

The findings in \citet{xu2021neural} raise a natural question; which pooling functions should be used for a given problem in order to make GNNs constructed with them successfully extrapolate for out-of-distribution data? \citet{xu2021neural} did not present any guide but empirically showed that we do have the ``right” pooling function for each problem tested, and when a pooling function is not properly selected, GNNs completely fails to extrapolate. The caveat here is that we do not know which pooling function is the right choice before actually training and validating the model.

To this end, in this paper, we present a generic learning-based method to find proper pooling functions for a given arbitrary problem. Our method, entitled Generalized Norm-based Pooling (GNP), formulates the pooling functions as a generic $L^p$ norm-like function (including negative $p$ as well), and learns the parameters inside the pooling functions in an end-to-end fashion. Unlike previous learning-based pooling methods that are usually tailored for specific tasks or focused on improving predictive performances, GNP can be applied to arbitrary tasks, and it improves the extrapolation ability of GNNs constructed with it. Also, GNP includes most of the pooling functions being used for GNNs as special cases. Despite the enhanced flexibility, GNP incurs minimal overhead in GNN in terms of the model complexity. A na\"ive application of GNP to GNNs is likely to fail because of some difficulty in training, so we propose a simple remedy to this. Using nine graph-level, node-level, and set-related tasks, we demonstrate that GNNs with GNP trained by our training scheme extrapolate 
for out-of-distribution data comparably and sometimes even better than those with pooling functions that are carefully chosen among of widely-used ones. We summarize our contributions as follows:
\begin{itemize}[leftmargin=10pt,noitemsep,topsep=0pt]
    \item {\bf Generalized pooling function}: We propose GNP, a simple yet flexible pooling function that can readily be applied to arbitrary tasks involving GNNs, with minimal parameter overhead. 
    \item {\bf Effective training methods}: We propose effective training methods for GNP.
    \item {\bf Extensive experiments}: We empirically demonstrate that GNNs with GNP can indeed generalize to out-of-distribution data on nine tasks.
\end{itemize}

%% file: 020related.tex
\paragraph{Aggregation functions}
Various aggregation functions have been appeared to enhance the performance of GNNs.
\cite{hamilton2017inductive} proposed GraphSAGE with four different aggregation methods; \texttt{max}, \texttt{mean}, GCN~\citep{kipf2017semi}, and LSTM~\citep{HochSchm97}. 
\cite{velivckovic2017graph} proposed Graph Attention neTworks (GATs) including attention-based aggregation functions~\citep{vaswani2017attention}.
\cite{xu2019powerful} proposed Graph Isomorphism Networks (GINs) and proved that GNN can satisfy the 1-Weisfeiler-Lehman (WL) condition
only with \texttt{sum} pooling function as aggregation function.
Recently, \cite{li2020deepergcn} proposed a trainable softmax and power-mean aggregation function that generalizes basic operators. 
Compared to these methods designed to improve interpolation performance on specific tasks, ours can improve extrapolation performance for generic tasks.

\paragraph{Readout functions}
\cite{zhang2018end} suggested SortPooling that chooses top-$k$ values from the sorted list of the node features to construct outputs.
Another popular idea is hierarchical pooling, where outputs are obtained by iteratively coarsening nodes in graphs in a hierarchical fashion~\citep{ying2018hierarchical, gao2019graph,lee2019self,yuan2020structpool}. Although demonstrated to be effective for the tasks they have been designed for, most of these methods require heavy computation and it is not straightforward to extend them for aggregation functions. On the other hand, our GNP can be applied to both aggregation and readout functions with minimal overhead.

\paragraph{Pooling functions in generic context}
\cite{vinyals2015order} proposed Set2Set to get a representation of set-structured data with a LSTM-based pooling function. \cite{lee2019set} proposed to use an attention-based pooling function to get summaries of set data. 
For convolutional neural networks, there were some approaches to generalize average pooling and max pooling widely used for many neural network architectures.
\cite{gulcehre2014learned} proposed a normalized learnable $L^p$ norm function that generalizes average pooling and max pooling. 
\cite{lee2016generalizing} further extended those pooling functions with learnable tree-structured pooling filters.

\paragraph{Norm-based pooling functions}
There have been several works to employ norm-based pooling functions. \citet{gulcehre2014learned} proposed a learnable $L^p$ norm function of the form
\[
    f(\mathbf{v}) = \left(\frac{1}{|\mathbf{v}|}\sum_{i=1}^{|\mathbf{v}|} |v_i|^p\right)^{1/p}
\]
to substitute max pooling or average pooling used in convolutional neural networks. Similar norm-based pooling functions were used for acoustic modeling~\citep{swietojanski2016differentiable} and text representation~\citep{wu2020attentive}. Compared to GNP, these pooling methods cannot express the \texttt{sum} pooling. \citet{li2020deepergcn} further generalized this by multiplying $|\mathbf{v}|^q$ to include \texttt{sum} pooling as well, but not considered the case where $p$ is positive and the case where $p$ is negative at the same time. GNP is the most generic norm-based pooling function, compared to all aforementioned approaches, and more importantly, no other works studied their usefulness in the context of learning to extrapolate.

\paragraph{Extrapolation}
\cite{trask2018neural} pointed out that most of the feed-forward neural networks fail to extrapolate even for the simplest possible identity mapping, and suggested using alternative computation units mimicking the behavior of arithmetic logic units. The ability to extrapolate is also important in the GNN context, for instance, many combinatorial optimization problems involving graphs often require extrapolation. \cite{selsam2018learning, prates2019learning} tackled the extrapolation problem by performing large iterations of message passing. Using various classical graph algorithms, \cite{velickovic2020Neural} showed that the extrapolation performance of GNNs depends heavily on the choice of the aggregation function. Similarly, \cite{xu2021neural} demonstrated that choosing the right non-linear function for both MLPs and GNNs is crucial for the extrapolation.

%% file: 040method.tex
In this section, we present our Generalized Norm-based Pooling (GNP) and discuss its expressiveness. 
Then, we describe some difficulties in training GNP and our remedy. Lastly, we present a task on which a GNN with \method can extrapolate, while that equipped with the basic pooling functions  cannot.

\subsection{Generalization of Basic Pooling Functions}

While \method is motivated by the $L^p$-norm function, which includes the \texttt{sum} and \texttt{max} functions as special cases, %
further ingredients are added to make \method more flexible than the $L^p$-norm function. Specifically, we allow $p$ to be negative to let \method express a wider class of functions than the previous norm-based or learning-based pooling functions.

Let $\mathbf{V} = \{\mathbf{v}_i\}_{i=1}^n$ be a set of node features with $\mathbf{v}_i \in \mathbb{R}^d$ for $i=1,\dots, n$. We define GNP to be an element-wise function where the output for each $j$th element is
\[\label{eq:gnp}
\method_j(\mathbf{V}) = \frac{1}{n^q} \left(
\sum_{i=1}^n |v_{i,j}|^p
\right)^{1/p},
\]
where $p \in \mathbb{R}\setminus\{0\}$ and $q \in \mathbb{R}$ are learnable parameters.
\method includes the basic pooling functions (\texttt{sum}, \texttt{mean}, \texttt{max}, and \texttt{min}) as special cases.

\begin{proposition}
\label{proposition:generalizing}
Suppose all the entries of $\mathbf{v}$ are non-negative in \eqref{eq:gnp}. Then, GNP includes \texttt{sum, max, min} as special cases. If we further restrict $\mathbf{v}$ to be positive, GNP includes \texttt{min}.
\end{proposition}

\begin{proof}
$\method_j(\mathbf{V})$ is equivalent to elementwise \texttt{sum} when $(p,q)=(1,0)$ and elementwise \texttt{mean} when $(p,q)=(1,1)$. When $q=0$, we have
\begin{align}
\lim_{p\to\infty} \method_j(\mathbf{V}) = \max_i v_i \lim_{p\to\infty} \left( \sum_{i=1}^{n}\left(\frac{|v_{i,j}|}{\max_i |v_{i,j}|}\right)^p\right)^{1/p} = \max_i |v_{i,j}| \cdot 1 = \max_i |v_{i,j}|,    
\end{align}
so GNP converges to \texttt{max}. Similarly, we can obtain \texttt{min} as a limit for $p \to -\infty$. \qedhere

\end{proof}

\subsection{\texorpdfstring{Handling of Negative $p$}{Handling negative p}}
The GNP function in \eqref{eq:gnp} is not continuous and even not defined at $p=0$. Hence, directly learning GNP in the original form as in \eqref{eq:gnp} (even with $p=0$ ignored) can cause instability, especially when an algorithm is trying to move from a positive $p$ value to a negative $p$ value. Instead, we suggest splitting the GNP function into two parts, $\method^+$ with positive $p$ and $\method^-$ with negative $p$, and let the model choose the right balance between them. Specifically, define
\begin{align}
\method^+_j(\mathbf{V}) = \frac{1}{n^{q^{+}}}\left(\sum_{i=1}^{n} |v_{i,j}|^{p^{+}}\right)^{1/p^{+}}, \quad 
\method^-_j(\mathbf{V}) = \frac{1}{n^{q^{-}}}\left(\sum_{i=1}^{n} |v_{i,j}|^{-p^{-}}\right)^{-1/p^{-}}, 
\end{align}
where $p^+>0$, $q^+$, $p^->0$, and $q^-$ are learnable parameters. Given a set of node features $\mathbf{V}$, we first split the feature dimension into two, and compute the output from $\method^+$ for the first half and from $\method^-$ for the second half. Then we mix two outputs with a single linear layer to get the final output.
\begin{align}
\mathbf{y} & = \begin{bmatrix}     
\method^+_1(\mathbf{V}) & \dots & \method^+_{\floor{d/2}}(\mathbf{V}) & 
\method^-_{\floor{d/2}+1}(\mathbf{V}) & \dots & \method^-_d(\mathbf{V})
\end{bmatrix}, \\
\method(\mathbf{V}) & = \mathbf{W}\mathbf{y} + \mathbf{b},
\end{align}
where {$\floor\cdot$ is the floor function, }$\mathbf{W}\in \mathbb{R}^{d\times d}$ and $\mathbf{b} \in \mathbb{R}^d$ are additional parameters to be learned. With this design, GNP can easily switch between positive $p$ and negative $p$, choosing proper values according to tasks. %

\subsection{Stabilization of Training Processes}
Unfortunately, even with the above design to split the positive and negative parts, GNP still suffers from a training instability issue. In this section, we introduce our remedy for such an issue. With our remedy, as we will empirically demonstrate, GNP can be applied to arbitrarily complex deep GNNs as a drop-in replacement for the existing pooling functions.

\paragraph{Negative or near-zero inputs}
GNP first processes inputs to be non-negative values by taking absolute values. In practice, in many GNN architectures, inputs are passed through ReLU before being fed into the pooling functions, so in such a case, we do not explicitly take the absolute values. If not, we explicitly put the ReLU activation function before every GNP to make inputs non-negative.

For the positive part of GNP, when the inputs are close to zero, the gradient w.r.t. the parameter $p^+$ may be exploded, as one can see from the following equation.
\begin{align*}
       \frac{\partial \method^+_j(\mathbf{V})}{\partial p^+} = 
       \frac{\method^+_j(\mathbf{V})}{p^+} \left( -\log(\method_j^+(\mathbf{V})) + \frac{\sum_{i=1}^{n} v_{i,j}^{p^+} \log(v_{i,j}) }{\method^+_j(\mathbf{V})^{p^+}} \right).
\end{align*}
Hence, we add a small tolerance term $\epsilon$ to every input element to prevent gradient explosion. This works well for positive $p$, but we need more care for negative $p$. When $p$ is negative, even small $\epsilon$ can be amplified by the term $(v_{i,j} + \epsilon)^{p^-}$ to dominate the other values. Hence, when a specific input $v_{i,j}$ is smaller than $\epsilon$, we replace it with $1/\epsilon$ to mask out the effect of that input for the output computation. An exceptional case is when every input element is below $\epsilon$. For such a case, we fix the output of GNP to be zero by default.
\begin{align*}
\tilde v_{i,j} &= \left\{
    \begin{array}{ll} 
     v_{i,j} + \epsilon & \text{ if } v_{i,j} > \epsilon \\
     1/\epsilon & \text{ otherwise }
\end{array} \right., \\
\method_j^-(\tilde{\mathbf{V}}) &=
\left\{
\begin{array}{ll}
0 & \text{ if } v_{i,j} < \varepsilon \text{ for } i=1,\dots, n \\
\frac{1}{n^{q-}}\left( \sum_{i=1}^n {\tilde v_{i,j}}^{p^-} \right)^{1/p^-} & \text{ otherwise }
\end{array}.
\right.
\end{align*}

Even with these treatments, still, the algorithm can diverge especially when $p$ is large. To resolve this, we clipped all $p$ values to be contained in $[0, 50]$ and used the log-sum-exp trick. That is, %
\[
\mathrm{GNP}_j^{p^+}(\mathbf{V}) = \frac{1}{n^{q^+}}\exp \left(
\frac{1}{p^+} \log \left( \sum_{i=1}^n \exp( p^+\log (v_{i,j} + \epsilon))
\right)
\right).
\]
Also, similar to \citet{gulcehre2014learned}, we reparameterized $p^+$ and $p^-$ with the softplus activation function, i.e., $p^+ = 1 + \log(1 + \exp(t^+))$ for some $t^+ \in \mathbb{R}$. 

Another important trick was to use different learning rates for training $(p^+, p^-)$ and $(q^+, q^-)$. Since the parameters $(q^+, q^-)$ have much larger impact on the GNP, if we use the same learning rates for  $(p^+, p^-)$ and $(q^+, q^-)$, the model can converge to unwanted local minimum that are not faithfully tuned for $(p^+, p^-)$. Hence, we used larger learning rates for $(p^+, p^-)$ to balance training. %

\subsection{Extrapolation Ability of \method}
\label{sec:method:analysis}

As stated in Theorem~\ref{thm:informal},
we prove that a GNN equipped with \method can extrapolate on the \texttt{harmonic} task, which we define in Section \ref{sec:exp:graph}. However, that equipped with the basic pooling functions cannot extrapolate on the task, as we show empirically in Section~\ref{sec:exp:graph} and theoretically in Appendix \ref{sec:app:theory}. %

\begin{theorem}
\label{thm:informal}
(Informal) Assume all the nodes in $G$ have the same scalar feature $1$. Then, a one-layer GNN equipped with \method and trained with squared loss in the NTK regime learns the \texttt{harmonic} task function, and thus it can extrapolate.
\end{theorem}
\begin{proof}
 See Appendix \ref{sec:app:theory} for detailed analysis.
\end{proof}

%% file: 050experiments.tex
In this section, we review our experiments on various extrapolation tasks. %

\subsection{Experimental Setups}
\label{sec:exp:setups}

\paragraph{Machines} We performed all experiments on a Linux server with RTX 3090 GPUs. 

\paragraph{GNN models} 
For graph-level tasks, 
we used one GIN \citep{xu2019powerful} layer with a hidden dimension of $32$ and two FC layers as MLP, and we fed only the outputs of the GIN layer into the readout function. 
Note that this simple model is expressive enough for obtaining exact answers to all considered graph-level tasks.
For node-level tasks, we used three of the aforedescribed GIN layers, without readout functions, so that nodes at most three hops away from the target node can be taken into consideration.
For set-level tasks, we used one FC layer with a hidden dimension of $32$ before the pooling function and used another FC layer for the final output after the pooling function.

\paragraph{Baseline}
Commonly for all tasks, we considered \texttt{sum}, \texttt{max}, \texttt{mean}, and \texttt{min}, all of which are generalized by \method, as baseline aggregation and/or readout functions.
For graph-level tasks, we additionally considered SortPooling \citep{zhang2018end} with $k=20$ and Set2Set \citep{vinyals2015order} as baseline readout functions, and we considered the hierarchical pooling version of SAGPool \citep{lee2019self} as a whole as a baseline model.
For set-level tasks, we additionally considered Set2Set \citep{vinyals2015order} as a baseline pooling function and Set Transformer \citep{lee2019set} as a whole as a baseline model.

\paragraph{Evaluation}
We compared evaluation metrics on the test set when validation loss was minimized, and in each setting, we reported mean and standard deviation over $5$ runs, unless otherwise stated.

\begin{figure}
    \vspace{-2mm}
    \centering
    \includegraphics[width=0.8\linewidth]{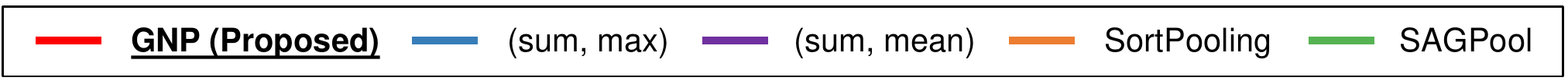}
    \\
    \begin{subfigure}{0.33\linewidth}
        \includegraphics[width=\linewidth]{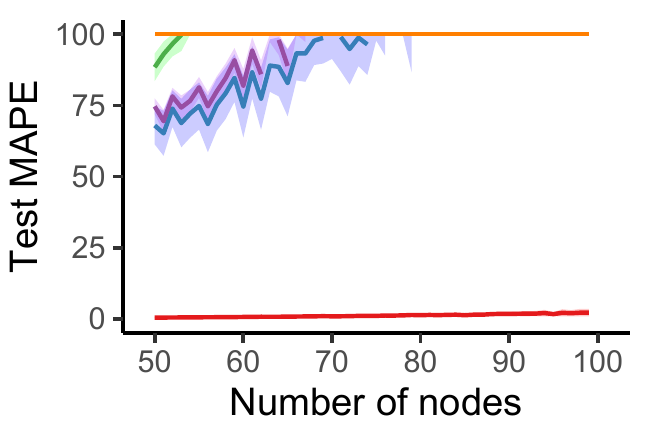}
        \caption{\texttt{invsize}}
    \end{subfigure}
    \hspace{-1.5mm}
    \begin{subfigure}{0.33\linewidth}
        \includegraphics[width=\linewidth]{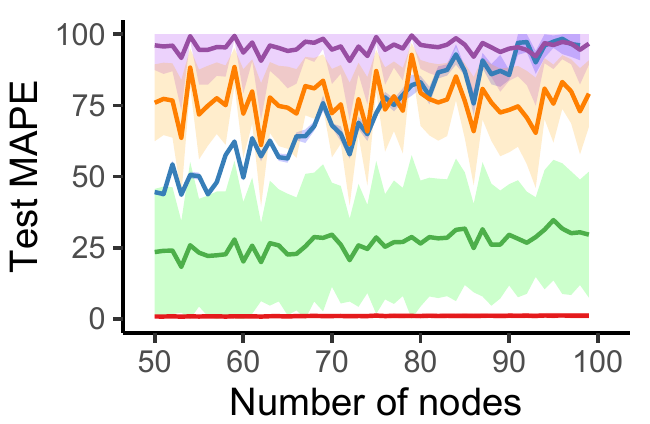}
        \caption{\texttt{harmonic}}
    \end{subfigure}
    \hspace{-1.5mm}
    \begin{subfigure}{0.33\linewidth}
        \includegraphics[width=\linewidth]{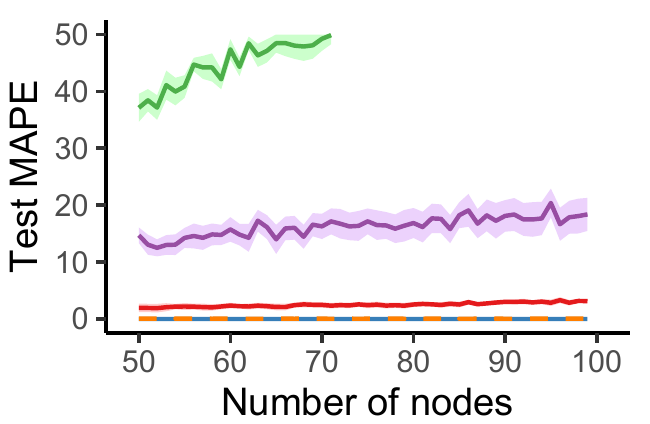}
        \caption{\texttt{maxdegree}}
    \end{subfigure}
    \caption{Extrapolation performances depending on the number of nodes in test graphs on three tasks (\texttt{invsize}, \texttt{harmonic}, and \texttt{maxdegree}). 
    Only GIN equipped with \method performed consistently well on all the tasks.
    We tested $19$ competitors and reported the results of the most successful ones.
    } %
    \label{fig:graph_ext}
\end{figure}

\begin{table}
  \vspace{-2mm}
  \caption{Extrapolation performance in terms of MAPE on large graphs with different structures.
  On two tasks (\texttt{invsize} and \texttt{harmonic}), \method significantly outperformed the second best one.
  }%
  \centering
  \scalebox{0.85}{
  \begin{tabular}{c|cccccc}
    \toprule
    \multirow{3}{*}{Types} & \multicolumn{2}{c}{\texttt{invsize}} & \multicolumn{2}{c}{\texttt{harmonic}} & \multicolumn{2}{c}{\texttt{maxdegree}} \\
    \cmidrule(r){2-7}
    & \multirow{2}{*}{\method} & Best Baseline & \multirow{2}{*}{\method} & Best Baseline & \multirow{2}{*}{\method} & Best Baseline \\
    & & (\texttt{sum}, \texttt{max}) & & (SAGPool) & & (\texttt{sum}, \texttt{max}) \\
    \midrule
    \texttt{BA}  & \textbf{0.9$\pm$0.3} & 92.5$\pm$10.5 & \textbf{2.5$\pm$0.9} & 78.4$\pm$40.8 & \underline{2.1$\pm$1.1} & \textbf{0.0$\pm$0.0} \\
    \texttt{Expander}  & \textbf{1.9$\pm$1.0} & \underline{35.4$\pm$7.8} & \textbf{0.9$\pm$0.5} & \underline{11.9$\pm$18.6} & \underline{2.3$\pm$1.1} & \textbf{0.0$\pm$0.0} \\
    \texttt{4regular} & \textbf{0.8$\pm$0.3} & 205.6$\pm$36.8 & \textbf{1.9$\pm$1.3} & 1179.3$\pm$310.6 & \underline{3.4$\pm$3.7} & \textbf{0.0$\pm$0.0} \\
    \texttt{Tree} & \textbf{0.8$\pm$0.3} & 202.3$\pm$11.9 & \textbf{14.7$\pm$6.3} & 149.4$\pm$34.9 & \underline{1.9$\pm$0.6} & \textbf{0.0$\pm$0.0} \\
    \texttt{Ladder} & \textbf{0.8$\pm$0.3} & 195.4$\pm$53.6 & \textbf{2.4$\pm$2.4} & 1138.4$\pm$283.3 & \underline{30.7$\pm$16.9} & \textbf{0.1$\pm$0.1} \\
    \bottomrule
  \end{tabular}
  \label{tab:exp:graph_diff}
  }
\end{table}

\subsection{Extrapolation Performances on Graph-level Tasks}
\label{sec:exp:graph}

In this section, we consider three graph-level tasks. Given a graph, the first task is to find the maximum node degree (\texttt{maxdegree}), and the second task is to compute the harmonic mean node degree divided by the number of nodes (\texttt{harmonic}). %
The last task is to compute the inverse of the number of nodes (\texttt{invnode}), which does not depend on the topology of the given graph. For details of the synthetic datasets we used, see Appendix \ref{sec:app:details:graph}.

For \texttt{maxdegree}, whose objective is $\textcolor{red}{\max}_{v\in V}\left(\textcolor{blue}{\sum}_{u\in N(v)}1\right)$, where $N(v)$ is the set of neighbors of $v$, the reasonable choice is to use \blue{\texttt{sum}} and \red{\texttt{max}} as aggregation and readout functions, respectively,
For \texttt{harmonic}, whose objective is $(\textcolor{brown}{\sum}_{v\in V}(\textcolor{blue}{\sum}_{u\in N(v)}1)^{\textcolor{brown}{-1}})^{\textcolor{brown}{-1}}$
the reasonable combination of aggregation and readout functions are \blue{\texttt{sum}} and \brown{\method~with $(p, q) = (-1, 0)$}, respectively.
For \texttt{invnode}, whose objective is $(\textcolor{brown}{\sum}_{v\in V}1^{\textcolor{brown}{-1}})^{\textcolor{brown}{-1}}$, any of \texttt{mean}, \texttt{max}, and \texttt{min} is reasonable as the aggregation function, and  \brown{\method~with $(p, q) = (-1, 0)$} is reasonable as the readout function.

We trained all models for $200$ epochs, %
and we compared their test MAPE\footnote{MAPE scales the error by the actual value, and it has been considered as a proper measure of extrapolation performance \citep{xu2021neural}.}  for evaluation in Figure~\ref{fig:graph_ext}. 
GIN with \method showed near-perfect extrapolation performances on all three tasks, and especially for \texttt{harmonic} and  \texttt{invnode}, GIN with \method was the only successful model.
Among the $16$ combinations of \texttt{sum}, \texttt{max}, \texttt{mean}, and \texttt{min}, using \texttt{sum} and \texttt{max} as the aggregation and readout functions, respectively, showed near-perfect extrapolation performance on \texttt{maxdegree}.
For the same task, another combination $(\texttt{mean}, \texttt{max})$ showed reasonably good performance.
For the other tasks, however, none of the $16$ combinations was successful.
SortPool and Set2Set as the readout function were tested, while fixing the aggregation function to the aforementioned reasonable one for each task.
While they performed almost perfectly for \texttt{maxdegree}, they failed at the other tasks.  
Lastly, SAGPool was not successful in any of the tasks.

We also tested the extrapolation performance  using large test graphs with distinctive structures.
As seen in Table~\ref{tab:exp:graph_diff}, GIN with \method showed near-perfect performance only except for \texttt{harmony} on random trees, and \texttt{maxdegree} on ladder graphs.
Especially, on \texttt{invsize} and \texttt{harmony}, it significantly outperformed the best baseline. We further tested the extrapolation performance of GNP and the baseline approaches using real-world graphs in Appendix \ref{sec:app:exp:realgraph}, graphs with different structures in Appendix \ref{sec:app:exp:structures}, graphs with different node feature distributions in Appendix \ref{sec:app:exp:nodefeats}, and various activation functions in  Appendix \ref{sec:app:exp:activations}.

\begin{table}
  \vspace{-2mm}
  \caption{Extrapolation performances in terms of MAE on two node-level tasks (\texttt{shortest} and \texttt{bfs}). GNP and all baseline methods were near perfect on \texttt{bfs}, and GNP was second best on \texttt{shortest}.}
  \centering
  \begin{subtable}{\linewidth}
      \centering
      \scalebox{0.85}{
      \begin{tabular}{c|ccccc}
        \toprule
        Aggregation & \texttt{sum} & \texttt{max} & \texttt{mean} & \texttt{min} & \method \\
        \midrule
        \texttt{bfs}       & \textbf{0.000$\pm$0.000} & \textbf{0.000$\pm$0.000} & \textbf{0.000$\pm$0.000} & \textbf{0.000$\pm$0.000} & \textbf{0.000$\pm$0.001} \\
        \texttt{shortest}  & 1.323$\pm$0.162 & 0.762$\pm$0.395 & 1.316$\pm$0.330 & \textbf{0.141$\pm$0.007} & \underline{0.332$\pm$0.105} \\
        \bottomrule
      \end{tabular}
      }
      \caption{Extrapolation Performance on Large Graphs with Homogeneous Structures.}
  \end{subtable}
  \begin{subtable}{\linewidth}
      \centering
      \scalebox{0.85}{
      \begin{tabular}{c|cc|ccc}
        \toprule
        \multirow{2}{*}{Types} & \multicolumn{2}{c|}{\texttt{bfs}} & \multicolumn{3}{c}{\texttt{shortest}}\\
        \cmidrule(r){2-6}
        & \method & \texttt{max} & \method & \texttt{min} & \texttt{max} \\
        \midrule
        \texttt{BA}  & \underline{0.001$\pm$0.001} & \textbf{0.000$\pm$0.000} & \underline{0.546$\pm$0.168} & \textbf{0.275$\pm$0.015} & 1.268$\pm$0.642 \\
        \texttt{Expander} & \textbf{0.000$\pm$0.000} & \textbf{0.000$\pm$0.000} & \underline{0.159$\pm$0.068} & \textbf{0.019$\pm$0.004} & 0.334$\pm$0.225 \\
        \texttt{4regular} & \underline{0.003$\pm$0.003} & \textbf{0.000$\pm$0.000} & \underline{1.911$\pm$0.257} & \textbf{1.188$\pm$0.182} & 5.178$\pm$1.218 \\
        \texttt{Tree} & \underline{0.003$\pm$0.002} & \textbf{0.000$\pm$0.000} & \underline{1.579$\pm$0.289} & \textbf{1.057$\pm$0.256} & 4.584$\pm$0.980 \\
        \texttt{Ladder} & \underline{0.002$\pm$0.001} & \textbf{0.000$\pm$0.000} & \underline{1.217$\pm$0.278} & \textbf{0.701$\pm$0.160} & 3.400$\pm$1.056 \\
        \bottomrule
      \end{tabular}
      }
      \caption{Extrapolation Performance on Large Graphs with Heterogeneous Structures.}
  \end{subtable}
  \label{tab:exp:node}
  \vspace{-2mm}

\end{table}

\subsection{Extrapolation Performance on Node-level Tasks}
\label{sec:exp:node}

We further evaluated the extrapolation performance of GNP on two node-level tasks considered in \cite{velickovic2020Neural}.
The first task is to decide whether each node is within $3$ hops from the target node or not. (\texttt{bfs}). We formulate the task as a regression problem and the label is $1$ within $3$ hops and $0$ outside $3$ hops.
The second task is to find the minimum distance from each node to the target node on a graph with non-negative weights (\texttt{shortest}).
Only the nodes within $3$ hops from the target node were taken into consideration.
As discussed in \citep{velickovic2020Neural}, one of the optimal models for the tasks imitates the parallel breadth-first search and the parallel Bellman-Ford algorithm \citep{bellman1958routing} for \texttt{bfs} and \texttt{shortest}, respectively. In such cases, the reasonable aggregators for \texttt{bfs} and \texttt{shortest} are \texttt{max} and \texttt{min}, respectively.

We considered five GINs equipped with \texttt{sum}, \texttt{max}, \texttt{mean}, and \texttt{min}, and GNP, respectively, as aggregation functions.
Note that the readout operation is not used for node-level tasks. For description of the datasets, see Appendix \ref{sec:app:details:node}.
We trained all of them for $100$ epochs for \texttt{bfs} and for $200$ epochs for \texttt{shortest};
and we compared their test MAE\footnote{MAPE was not applicable since the ground-truth value for some nodes can be $0$.} in Table~\ref{tab:exp:node}.
GNP and all baseline methods were near perfect on \texttt{bfs}, regardless of graph types, and GNP was second best on \texttt{shortest}. As expected, GIN with \texttt{min} performed best on \texttt{shortest}.

\subsection{Extrapolation Performance on Set-related Tasks}
\label{sec:exp:set}

We also applied our proposed approach to three set-related tasks.
They are all related to estimating posterior distributions when the likelihood function is Gaussian.
Specifically, the tasks are to find closed-form posterior hyperparameters $\mu_{\text{post}}$ and $\sigma^2_{\text{post}}$, the MAP estimate $\hat{\mu}_{\text{MAP}}$ of $\mu$ when $\sigma^2$ is known, and the MAP estimate $\hat{\sigma}^2_{\text{MAP}}$ if $\sigma^2$ when $\mu$ is known. 
Note that ground-truth values of $\mu_{\text{post}}$ and  $\hat{\mu}_{\text{MAP}}$ are identical, while we used different loss functions for them. For description of the datasets, see Appendix \ref{sec:app:details:set}.

\begin{table}
\vspace{-4mm}
  \caption{Extrapolation performance in terms of MAPE on set-related tasks. %
  Only the basic model equipped with \method performed consistently well on all tasks.
  Especially for $\sigma^2_{\text{post}}$ and $\hat{\sigma}^2_{\text{MAP}}$, it significantly outperformed all competitors, including Set Transformer. %
  }
  \centering
  \scalebox{0.85}{
  \begin{tabular}{llcccc}
    \toprule
    Model & Pooling & $\mu_{\text{post}}$ & $\sigma^2_{\text{post}}$ & $\hat{\mu}_{\text{MAP}}$ & $\hat{\sigma}^2_{\text{MAP}}$  \\
    \midrule
    \multirow{5}{*}{Basic}
    & \texttt{sum}  & 135.0 $\pm$ 9.3 & 390.7 $\pm$ 99.1 & 126.8 $\pm$ 18.2 & 369.1 $\pm$ 12.1 \\
    & \texttt{max}  & 119.2 $\pm$ 31.7 & 120.6 $\pm$ 4.0 & 118.6 $\pm$ 31.6 & 108.8 $\pm$ 2.0 \\
    & \texttt{mean} & \underline{1.9 $\pm$ 0.2} & 134.2 $\pm$ 6.0 & \underline{1.9 $\pm$ 0.2} & 107.1 $\pm$ 2.2 \\
    & \texttt{min} & 95.8 $\pm$ 15.6 & 126.2 $\pm$ 3.9 & 118.6 $\pm$ 31.6 & 108.0 $\pm$ 2.4 \\
    & Set2Set  & \underline{2.1 $\pm$ 0.2} & 135.7 $\pm$ 4.0 & \underline{1.9 $\pm$ 0.2} & 106.1 $\pm$ 2.6 \\
    \midrule
    \multirow{5}{*}{Deep} & \texttt{sum} & 136.3 $\pm$ 6.8 & 119.3 $\pm$ 16.1 & 100.0 $\pm$ 0.0 & 381.9 $\pm$ 10.8  \\
    & \texttt{max}  & 100.0 $\pm$ 0.0 & 123.8 $\pm$ 2.1 & 98.9 $\pm$ 2.6 & 109.6 $\pm$ 3.3 \\
    & \texttt{mean} & \underline{2.2 $\pm$ 0.2} & 135.2 $\pm$ 4.0 & \underline{2.2 $\pm$ 0.4} & 109.2 $\pm$ 2.9 \\
    & \texttt{min} & 83.0 $\pm$ 10.2 & 99.5 $\pm$ 2.1 & 90.8 $\pm$ 6.0 & 108.8 $\pm$ 5.3 \\
    & Set2Set  & \underline{1.9 $\pm$ 0.3} & 131.1 $\pm$ 8.5 & \underline{1.9 $\pm$ 0.2} & 106.0 $\pm$ 1.5 \\
    \midrule
    \multicolumn{2}{l}{Set Transformer}  & \underline{1.9 $\pm$ 0.2} & 25.0 $\pm$ 9.0 & \underline{1.9 $\pm$ 0.1} & 40.8 $\pm$ 9.5 \\
    \midrule
    \textbf{Basic} & \textbf{\method} & \textbf{1.5 $\pm$ 0.6} & \textbf{0.7 $\pm$ 0.3}  & \textbf{1.5 $\pm$ 0.6} & \textbf{3.1 $\pm$ 0.5} \\
    \bottomrule
  \end{tabular}
  }
  \label{tab:exp:set}
\end{table}

We trained for $300$ epochs (a) the basic model (see Section~\ref{sec:exp:setups}) with GNP, (b) Set Transformer \citep{lee2019set} (c) the basic and deep\footnote{The deep model has an additional FC layer before the pooling function.} models with one among \texttt{sum}, \texttt{max}, \texttt{mean}, \texttt{min}, and Set2Set \citep{vinyals2015order}.
We compared their MAPE in Table~\ref{tab:exp:set}.
The basic model equipped with \method showed near-perfect extrapolation performance on all four tasks, even though the formula for $\hat{\sigma}^2_{\text{MAP}}$ cannot be exactly expressed by \method, and it was the only such model. 
For $\mu_{\text{post}}$ and $\hat{\mu}_{\text{MAP}}$, whose ground-truth values are approximated by the average of the elements, Set Transformer and those equipped with \texttt{mean} or Set2Set were comparable to the basic model with \method, while they were not on the other tasks.

\begin{figure}
    \vspace{-4mm}
    \centering
       \begin{subfigure}{0.25\linewidth}
        \includegraphics[width=\linewidth]{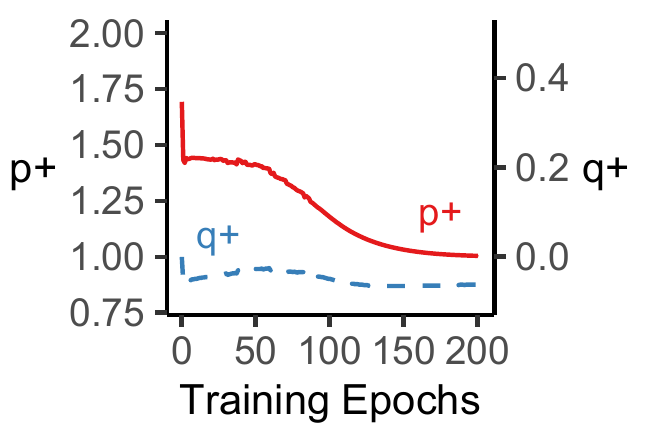}
        \caption{\texttt{sum} (\texttt{maxdegree})}
    \end{subfigure}
    \hspace{-2mm}
 \begin{subfigure}{0.25\linewidth}
        \includegraphics[width=\linewidth]{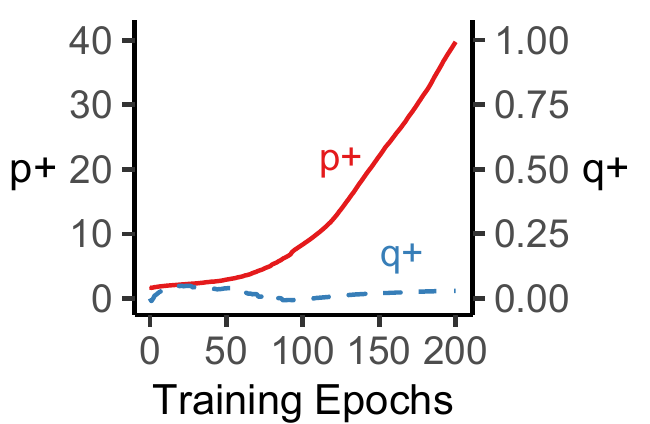}
        \caption{\texttt{max} (\texttt{maxdegree})}
    \end{subfigure}
    \hspace{-2mm}
    \begin{subfigure}{0.25\linewidth}
        \includegraphics[width=\linewidth]{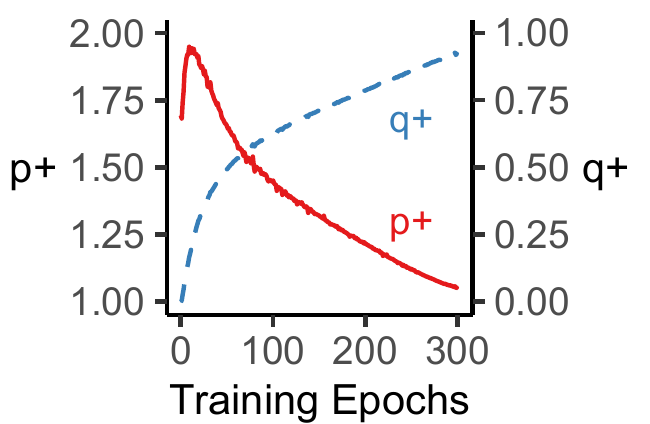}
        \caption{\texttt{mean} ($\mu_{\text{post}}$)}
    \end{subfigure}
    \hspace{-2mm}
    \begin{subfigure}{0.25\linewidth}
        \includegraphics[width=\linewidth]{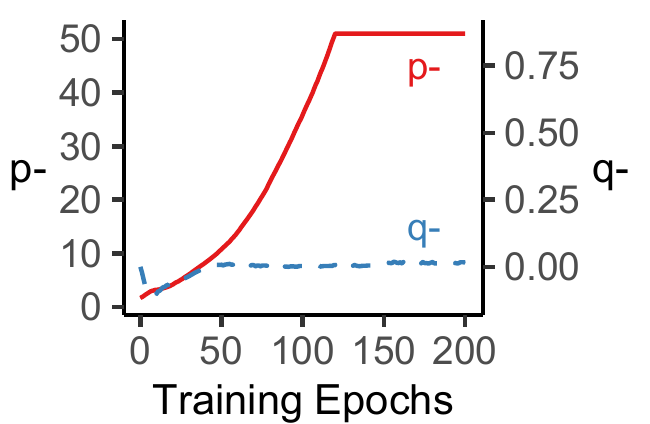}
        \caption{\texttt{min} (\texttt{shortest})}
    \end{subfigure}
    \caption{Empirical behavior of \method. We showed how the parameters $p$ and $q$ of \method changed during training. For each task, \method imitated the ideal pooling functions  if such pooling functions exist. 
    For example, for \texttt{maxdegree}, GNP as aggregation and readout functions approximated $\texttt{sum}$ (i.e., $p^+\approx 1$ and $q^+\approx 0$) and $\texttt{max}$ (i.e., $p^+ \gg 1$ and $q^+\approx 0$), respectively, which performed best.
    }
    \label{fig:sec:operators}
\end{figure}

\subsection{Empirical Behavior of \method}
\label{sec:exp:behavior}
As we discussed in Section \ref{sec:method}, GNP generalizes  \texttt{sum}, \texttt{max}, \texttt{mean}, and \texttt{min}. 
In order to confirm the facts experimentally, we showed in Figure~\ref{fig:sec:operators} how the learnable parameters $p$ and $q$ in GNP changed during training.
For \texttt{maxdegree},
GNP as aggregation and readout functions approximated $\texttt{sum}$ (i.e., $p^+\approx 1$ and $q^+\approx 0$) and $\texttt{max}$ (i.e., $p^+ \gg 1$ and $q^+\approx 0$), respectively, which performed best on the task. %
For $\mu_{\text{post}}$ and \texttt{shortest}, GNP approximated $\texttt{mean}$ (i.e., $p^+ \approx 1$ and $q^+\approx 1$) and $\texttt{min}$ (i.e., $p^- \gg 0$ and $q^-\approx 0$), respectively, which were the best performing baseline for the tasks.
To sum up, empirically, \method imitated the ideal pooling functions for each task if such pooling functions exist. We also observed that either $\method^{+}$ or $\method^{-}$ tends to dominate the other side in all considered graph-level tasks. Detailed results are provided in Appendix \ref{sec:app:exp:behaviors}.

\begin{table}
    \vspace{-2mm}
    \centering
    \caption{Effectiveness of $\method^{-}$. The extrapolation performance of \method degraded without $\method^{-}$.}
    \begin{subtable}[h]{0.48\linewidth}
    \centering
    \scalebox{0.85}{
        \begin{tabular}{ccc}
            \toprule
            Tasks & \method & $\method^{+}$ \\
            \midrule
            \texttt{harmonic} & \textbf{1.1 $\pm$ 0.8} & \underline{2.1 $\pm$ 0.6} \\
            \texttt{shortest} & \textbf{0.332 $\pm$ 0.105} & 0.774 $\pm$ 0.135 \\
            $\sigma^2_{\text{post}}$ & \textbf{0.7 $\pm$ 0.3} & \textbf{0.6 $\pm$ 0.2} \\
            \bottomrule
        \end{tabular}
        }
        \caption{Test Error on Erdős–Rényi Random Graphs}
    \end{subtable}
    \begin{subtable}[h]{0.48\linewidth}
    \centering
    \scalebox{0.85}{
        \begin{tabular}{ccc}
            \toprule
            Graphs & \method & $\method^{+}$ \\
            \midrule
            \texttt{BA} & \textbf{2.5 $\pm$ 0.9} & 31.5 $\pm$ 1.2 \\
            \texttt{tree} & \textbf{14.7 $\pm$ 6.3} & 26.1 $\pm$ 7.4 \\
            \texttt{ladder} &\textbf{2.4 $\pm$ 2.4} & 19.3 $\pm$ 21.1 \\
            \bottomrule
        \end{tabular}
        }
        \caption{Test Error on \texttt{harmonic} on Other Graphs }
    \end{subtable}
    \label{tab:ablation}
\end{table}

\subsection{Ablation Study: Effectiveness of $\method^{-}$}
\label{sec:exp:ablation}
In order to demonstrate the effectiveness of $\method^{-}$ for extrapolation,
we compared the model equipped only with \method and the model only with $\method^{+}$ on each of three tasks (\texttt{harmony}, \texttt{shortest}, and $\sigma^2_{\text{post}}$) in Table~\ref{tab:ablation}.
The detailed settings for each task were the same as in previous experiments. 
The model only with $\method^{+}$ performed well only on the task for $\sigma^2_{\text{post}}$.
The extrapolation performance of \method degraded significantly without $\method^{-}$ on \texttt{harmony} and \texttt{shortest}.

\begin{table}[t]
  \vspace{-4mm}
  \centering
  \caption{Graph classification accuracy. Replacing the carefully chosen pooling functions in SAGPool and ASAPool  with \method improved their accuracy on graph-classification tasks.}
  \begin{subtable}{\linewidth}
  \centering
  \scalebox{0.85}{
  \begin{tabular}{cccccc}
    \toprule
    Model & Aggregation & Readout & \texttt{D\&D} & \texttt{PROTEINS} & \texttt{NCI1} \\
    \midrule
    SAGPool (original) & GCN & mean, max & 0.765 $\pm$ 0.009 & 0.722 $\pm$ 0.008 & 0.688 $\pm$ 0.013 \\
    SAGPool (with \method) & \textbf{\method} & \textbf{\method} & \textbf{0.774 $\pm$ 0.010} & \textbf{0.728 $\pm$ 0.013} & \textbf{0.695 $\pm$ 0.015} \\
    \bottomrule
  \end{tabular}
  }
  \caption{SAGPool}
  \end{subtable}
  \begin{subtable}{\linewidth}
  \centering
  \scalebox{0.85}{
  \begin{tabular}{cccccc}
    \toprule
    Model & Aggregation & Readout & \texttt{D\&D} & \texttt{PROTEINS} & \texttt{NCI1} \\
    \midrule
    ASAPool (original) & GCN & mean, max & 0.764 $\pm$ 0.009 & 0.738 $\pm$ 0.008 & 0.711 $\pm$ 0.004 \\
    ASAPool (with \method) & \textbf{\method} & \textbf{\method} & \textbf{0.772 $\pm$ 0.007} & \textbf{0.739 $\pm$ 0.006} & \textbf{0.725 $\pm$ 0.007} \\
    \bottomrule
  \end{tabular}
  }
  \caption{ASAPool}
  \end{subtable}
  \label{tab:exp:graphcl}
\end{table}

\subsection{Effectiveness of \method on Two Real-world Tasks}
\label{sec:exp:graphcl}\label{sec:exp:im}

\paragraph{Graph classification} 
We compared the graph classification accuracy of hierarchical SAGPool \citep{lee2019self} and ASAPool \citep{ranjan2020asap}, and their variants with GNP. For the variant of SAGPool, we replaced all pooling functions before, inside, and between graph pooling operations. 
For the variant of ASAPool, we replaced all pooling functions except for those inside LEConv. Since we used \method, instead of the concatenation of global average pooling and max pooling functions, the input dimension of the first fully-connected layer after them was reduced by half. For the variants, except for the additional hyperparameters of \method, all hyperparemters were set the same as those in the original one.

We used three datasets from TUDataset \citep{morris2020tu}.  \texttt{D\&D} \citep{dobson2003distinguishing,shervashidze2011weisfeiler} and \texttt{PROTEINS} \citep{dobson2003distinguishing,borgwardt2005protein} contain protein-interaction graphs, and \texttt{NCI1} \citep{wale2006comparison} contains the graphs representing chemical compounds. %
For consistency with the original SAGPool, we performed $10$-fold cross validation with $20$ different random seeds. For ASAPool, we performed $10$-fold cross validation with the $20$ random seeds specified in its implementation.

We report the test accuracy with standard deviation in Table \ref{tab:exp:graphcl}. SAGPool and ASAPool equipped with \method consistently outperformed the original models with a carefully chosen pooling functions.

\begin{table}[]
    \vspace{-4mm}
    \centering
    \caption{Influence maximization performance. The influences of $100$ seed nodes produced by MONSTOR and its variants in graphs unseen during training are reported. The variant equipped with GNP outperforms  original MONSTOR (with \texttt{max}) and the other variant (with \texttt{sum}).}
    
    \scalebox{0.75}{
    \setlength{\tabcolsep}{0.3em}
    \begin{tabular}{cccccccccc}
        \toprule
        \multirow{2}{*}{Aggregation} & \multicolumn{3}{c}{\texttt{Extended}} & \multicolumn{3}{c}{\texttt{WannaCry}} & \multicolumn{3}{c}{\texttt{Celebrity}} \\
        & \texttt{BT} & \texttt{JI} & \texttt{LP} & \texttt{BT} & \texttt{JI} & \texttt{LP} & \texttt{BT} & \texttt{JI} & \texttt{LP} \\ 
        \midrule
        \texttt{max} & 1222.5$\pm$0.4 & 706.9$\pm$0.1 & 3259.6$\pm$0.7 & 2746.5$\pm$1.4 & 1646.6$\pm$2.1 & 9090.2$\pm$3.8 & 155.2$\pm$0.1 & \textbf{140.5$\pm$0.0} & 5665.0$\pm$1.4\\
        \texttt{sum} & 1216.6$\pm$1.7 & 706.5$\pm$0.2 & 3189.2$\pm$6.9 & 2742.6$\pm$0.9 & 1645.8$\pm$0.2 & 9030.1$\pm$2.0 & 153.9$\pm$0.4 & \textbf{140.5$\pm$0.0} & \textbf{5666.9$\pm$0.6} \\
        \textbf{\method} & \textbf{1223.0$\pm$0.3} & \textbf{707.3$\pm$0.2} & \textbf{3262.1$\pm$1.7} & \textbf{2753.4$\pm$0.1} & \textbf{1648.3$\pm$0.1} & \textbf{9098.4$\pm$2.2} & \textbf{155.3$\pm$0.8} & 140.4$\pm$0.0 & 5666.1$\pm$1.8 \\
        \bottomrule
    \end{tabular}
    }
    \label{tab:im_monstor}
\end{table}

\paragraph{Influence maximization}

We compared the performance of MONSTOR \citep{ko2020monstor} and its variants with GNP on the influence maximization task \citep{kempe2003maximizing}, which has been extensively studied due to its practical applications in viral marketing and computational epidemiology. 
The objective of the task is to choose a given number of seed nodes so that their collective influence (i.e., degree of spread of information through a given social network) is maximized.

For experimental details, we followed \citep{ko2020monstor}:
(a) we used three real-world social networks (\texttt{Extended}, \texttt{WannaCry}, and \texttt{Celebrity}) with three kinds of realistic activation probabilities (\texttt{BT}, \texttt{JI}, and \texttt{LP}), (b) we used the same training methods and hyperparameters except for the additional parameters of \method, and (c) we compared MONSTOR and its variants in an inductive setting. For example, we used the model trained using the \texttt{Celebrity} and \texttt{WannaCry} datasets to test the performance on the \texttt{Extended} dataset. For additional details of the influence maximization problem and MONSTOR, see Appendix \ref{sec:app:im}.

We performed three runs and reported the influence maximization performance with standard deviations in Table \ref{tab:im_monstor}. As seen in the results with \texttt{sum} and \texttt{max} aggregations, the performances heavily depended on the choice of the aggregation function. 
In most of the cases, MONSTOR equipped with \method outperformed the original MONSTOR with \texttt{max} aggregation and also a variant of MONSTOR with \texttt{sum} aggregation.

%% file: 060conclusion.tex
In this work, we proposed GNP, a learnable norm-based pooling function that can readily be applied to arbitrary GNNs or virtually to any neural network architecture involving permutation-invariant pooling operation. The key advantages of GNP are its generality and ability to extrapolate. We showed that GNP includes most of the existing pooling functions and can express a broad class of pooling functions as its special cases. More importantly, with various synthetic and real-world problems involving graphs and sets, we demonstrated that the networks with GNP as aggregation or readout functions can correctly identify the pooling functions that can successfully extrapolate. We also introduced some non-trivial design choices and techniques to stably train GNP. The limitation of our work is that, although we have empirically demonstrated the excellent extrapolation performance on various tasks, we have not developed theoretical arguments regarding under what condition models constructed with GNP will extrapolate well. It would be an interesting future work to rigorously study the class of problems that GNP can solve.

%% file: supple.tex
\newpage
\section{Theoretical Analysis (Related to Section~\ref{sec:method:analysis})}
\label{sec:app:theory}
Similarly to \cite{xu2021neural}, we present an analysis of the extrapolation ability of GNNs with our \method pooling function.
Specifically, we show that a one-layer GNN with the \method pooling function can extrapolate on the \texttt{harmonic} task. Let $f(\bm{\theta}, G)$ be a one-layer GNN defined as follows:
\begin{align}\label{eq:gnn_gnp_structure}
f(\bm{\theta}, G) = W^{(2)} \text{ } \method\left(\left\{\sum_{v\in N(u)} W^{(1)}\bm{x_v}\right\}_{u \in V}\right),
\end{align}
where $\bm{\theta}$ is the parameters of the GNN, $G = (V, E)$ is the input graph, $x_v$ is the initial feature of node $v \in V$, and $N(u) \subseteq V$ is the set of the neighbors of node $u \in V$.

When two graphs $G$ and $G'$ are given, the Graph Neural Tangent Kernel (GNTK)~\citep{du2019graph} is computed as 
$$
\textsc{GNTK}(G, G') = \mathbb{E}_{\bm{\theta} \sim \mathcal{N}(0, I)}\left[\left\langle\frac{\partial f(\bm{\theta}, G)}{\partial \bm{\theta}},\frac{\partial f(\bm{\theta}, G')}{\partial \bm{\theta}}\right\rangle\right].
$$

\subsection{Feature map of the GNTK}
We first compute the GNTK for the network defined as \eqref{eq:gnn_gnp_structure} and derive the corresponding feature map using a general framework presented in \citet{jacot2018neural, du2019graph,xu2021neural}. Let $\Sigma^{(1)}$, $\Sigma^{(2)}$ be the covariance for the first linear layer and the second linear layer, respectively. Also, let $\Theta^{(1)}$ be the kernel value after the first linear layer, respectively.

From the framework, $\Sigma^{(1)}$ and $\Theta^{(1)}$ are determined as follows:
\begin{align*}
    \left[\Sigma^{(1)}(G, G')\right]_{uu'} &= \left[\Theta^{(1)}(G, G')\right]_{uu'} = \bm{x}_u^\top\bm{x}_u.
\end{align*}

Also, $\Sigma^{(2)}$ can be computed as

\begin{align*}
    \Sigma^{(2)}(G, G') = \mathbb{E}_{(f(v), f(v')) \sim \mathcal{N}\left(\bm{0}, [\Lambda^{(1)}(G, G')]_{vv'}\right)}\Bigg[ & \method\left(\left\{\sum_{v \in N(u)} f(v) \right\}_{u \in V}\right)\nonumber\\
    & \method\left(\left\{\sum_{v' \in N(u')} f(v') \right\}_{u' \in V'}\right)\Bigg],
\end{align*}

where
$$[\Lambda^{(1)}(G, G')]_{vv'} = \begin{pmatrix}
  [\Sigma^{(1)}(G, G)]_{vv} & [\Sigma^{(1)}(G, G')]_{vv'}\\
  [\Sigma^{(1)}(G', G)]_{v'v} & [\Sigma^{(1)}(G', G')]_{v'v'}
\end{pmatrix} = \begin{pmatrix}
  \bm{x}_v^\top\bm{x}_v & \bm{x}_v^\top\bm{x}_v'\\
  \bm{x}_v'^\top\bm{x}_v & \bm{x}_v'^\top\bm{x}_v'
\end{pmatrix}.$$

By a simple algebraic manipulation, one can easily see that the feature map $\phi(G)$ is computed as

\begin{align}
\phi(G) &= c \cdot \Bigg(\method\left(\{{\bm{w}^{(k)}}^\top\bm{h_v}\}_{v \in V}\right), \nonumber\\
& \frac{1}{|V|^q} \sum_{u \in V} ({\bm{w}^{(k)}}^\top\bm{h_u})^{p-1} \bigg(\sum_{v \in V} ({\bm{w}^{(k)}}^\top\bm{h_v})^{p}\bigg)^{\frac{1}{p} - 1} \mathbb{I}({\bm{w}^{(k)}}^\top\bm{h_u} > 0) \cdot \bm{h_u}, \ldots\Bigg),
\end{align}

where $\bm{w}^{(k)} \sim \mathcal{N}(\bm{0}, \bm{I})$, $c$ is a constant, and $\bm{h}_{u}$ is the sum of the initial features of the neighbors $N(u)$ of node $u$, i.e. $\sum_{v \in N(u)} \bm{x}_v$.

\subsection{Analysis on the \texttt{harmonic} task}

We analyze the extrapolation ability of GNNs on the \texttt{harmonic} task, whose exact functional form is given as
\begin{align}
f^\star(G) = \Bigg(\sum_{v \in V} \bigg(\sum_{u \in N(v)} 1\bigg)^{-1}\Bigg)^{-1}.\label{eq:harmonic}
\end{align}
Following \citet{xu2021neural}, we assume \emph{linear algorithmic alignment}; if a neural network can simulate a target function $f$ by replacing MLP modules with \emph{linear} functions, (i.e., the nonlinearities of the neural network is well matches with the target function, so the neural network only has to learn the linear (MLP) part), than it can correctly learn the target function, and thus can extrapolate well.  With this hypothesis, we proceed as follows. We assume that a GNN is in the NTK regime, that is, the GNN is initialized in a specific way called NTK parameterization, trained via gradient descent with small step size, and the widths of the network tend to infinty. In such case, the GNN behaves as a solution to kernel regression with GNTK kernel. Then we convert the kernel regression problem into a constrained optimization problem in the feature space induced from GNTK kernel, and show that the solution for the constrained optimization problem aligns with the functional form of the \texttt{harmonic} task (\eqref{eq:harmonic}).

We first state the following Lemma from \citet{xu2021neural} showing that a NTK kernel regression solution can be viewed as a constrained optimization problem in the feature space.
\begin{lemma}[\textbf{Lemma 2} in \citet{xu2021neural}]\label{lem:kernel_regression}
Suppose $\text{NTK}_{\text{train}}$ is the $n \times n$ kernel for training data, $\text{NTK}(\bm{x}, \bm{x}_i)$ is the kernel value between test data $\bm{x}$ and training data $\bm{x}_i$, and $Y$ is the training labels.
Let $\phi(\bm{x})$ be a feature map induced by a neural tangent kernel, for any $x \in \mathbb{R}^d$. The solution to kernel regression 
$$
(\text{NTK}(\bm{x}, \bm{x}_1),\ldots,\text{NTK}(\bm{x}, \bm{x}_n)) \cdot \text{NTK}_{\text{train}}^{-1}Y
$$
is equivalent to $\phi(\bm{x})^\top\bm{\beta}_{\text{NTK}}$, where $\bm{\beta}_{\text{NTK}}$ is
\begin{align}
&\min_{\bm{\beta}} \Vert\bm{\beta}\Vert_{2} \nonumber\\
&\text{\text{s.t.} } \phi(\bm{x}_i)^\top\bm{\beta}=y_i, \quad \text{ \text{for} } i = 1,\ldots,n.
\label{eq:constrained_optimization}
\end{align}
\end{lemma}

\begin{proof}
See \cite{xu2021neural}.
\end{proof}

\begin{theorem}
Assume all the nodes in $G$ have the same scalar feature $1$. Then, a GNN defined as \eqref{eq:gnn_gnp_structure} trained with squared loss in the NTK regime learns the \texttt{harmonic} task function (\eqref{eq:harmonic}). 
\end{theorem}

\begin{proof}
Assume $(p^+, p^-, q^+, q^-) = (\infty, 1, \infty, 0)$. Then every output of $\method^{+}$ goes to zero regardless of inputs and $\method^{-}$ aligns with the target function \eqref{eq:harmonic}. The feature map of GNTK in this case can be simplified as follows:

\begin{align*}
\phi(G) &= c \cdot \Bigg(\method\left(\{{\bm{w}^{(k)}}^\top\bm{h_v}\}_{v \in V}\right), \nonumber\\
& \sum_{u \in V} ({\bm{w}^{(k)}}^\top\bm{h_u})^{-2} \left(\sum_{v \in V} \frac{1}{{\bm{w}^{(k)}}^\top\bm{h_v}}\right)^{-2} \mathbb{I}({\bm{w}^{(k)}}^\top\bm{h_u} > 0) \cdot \bm{h_u}, \ldots\Bigg).
\end{align*}

By \textbf{Lemma~\ref{lem:kernel_regression}}, we know that in the NTK regime, the GNN $f(\bm{\theta}, G)$ behaves as the solution to the constrained optimization problem~\eqref{eq:constrained_optimization} with feature map $\phi(G)$ and coefficients $\bm{\beta}$. Let $\hat{\bm{\beta}}_{\bm{w}} \in \mathbb{R}$ be a coefficient corresponding to $\method( \{ \bm{w}^\top \bm{h}_u \}_{u \in V} )$ and $\hat{\bm{\beta}}'_{\bm{w}} \in \mathbb{R}$ be a coefficient corresponding to the other term in $\phi(G)$. Similar to Lemma 3 in \cite{xu2021neural}, we can combine the effect of coefficients for $\bm{w}$'s in the same direction. For each $\bm{w} \sim \mathrm{Unif}(\text{unit sphere})$, we can define $\bm{\beta}_{\bm{w}}$ and $\bm{\beta}_{\bm{w}}'$ as the total effect of weights in the same direction with considering scaling.

\begin{align*}
    \bm{\beta}_{\bm{w}} &= \int \bm{\beta}_{\bm{u}} \mathbb{I}\left(\frac{\bm{w}^\top \bm{u}}{\Vert \bm{w} \Vert \Vert \bm{u} \Vert} = 1 \right) \cdot \frac{\Vert \bm{u} \Vert}{\Vert \bm{w} \Vert} \mathbb{P}(\bm{u}), \\
    \bm{\beta}'_{\bm{w}} &= \int \bm{\beta}'_{\bm{u}} \mathbb{I}\left(\frac{\bm{w}^\top \bm{u}}{\Vert \bm{w} \Vert \Vert \bm{u} \Vert} = 1 \right) \cdot \frac{\Vert \bm{u} \Vert}{\Vert \bm{w} \Vert} \mathbb{P}(\bm{u}).
\end{align*}

Since the dimension of the input features is $1$, we only need to consider two directions of $\bm{w}$. To get min-norm solution, we compute the Lagrange multiplier as

\begin{align*}
    &\min_{\hat{\bm{\beta}},\hat{\bm{\beta}}'} \int \hat{\bm{\beta}}_{\bm{w}}^{2} + \hat{\bm{\beta}}_{\bm{w}}'^{2} d\mathbb{P}(\bm{w}) \\
    \text{\textit{s.t. }} & \int\method\left(\{\bm{w}^\top\bm{h_v}\}_{v \in V}\right) \cdot \hat{\bm{\beta}}_{\bm{w}} + \sum_{u \in V} (\bm{w}^\top\bm{h_u})^{-2} \left(\sum_{v \in V} \frac{1}{\bm{w}^\top\bm{h_v}}\right)^{-2} \mathbb{I}(\bm{w}^\top\bm{h_u} > 0) \cdot \hat{\bm{\beta}}_{\bm{w}}' \cdot \bm{h_u} d\mathbb{P}(\bm{w}) \\
    & =\left(\sum_{u\in V_i}\bm{h}_u^{-1}\right)^{-1} \quad \forall i \in [n],
\end{align*}
where $G_i = (V_i, E_i)$ is the $i$-th training data and $\bm{w} \sim \mathcal{N}(0, 1)$. By KKT condition, taking the derivative for each variable, we can get the following conditions:

\begin{align*}
    \hat{\bm{\beta}}_{+} &= c \cdot \sum_{i=1}^{n} \lambda_i \cdot \left(\sum_{u\in V_i}\bm{h}_u^{-1}\right)^{-1}, \\
    \hat{\bm{\beta}}_{-} &= 0 \\
    \hat{\bm{\beta}}'_{+} &= c \cdot \sum_{i=1}^{n} \lambda_i \cdot \left(\sum_{u\in V_i}\bm{h}_u^{-1}\right) \cdot \left(\sum_{v\in V_i}\bm{h}_v^{-1}\right)^{-2} \\
    &= c \cdot \sum_{i=1}^{n} \lambda_i \cdot \left(\sum_{u\in V_i}\bm{h}_u^{-1}\right)^{-1}, \\
    \hat{\bm{\beta}}'_{-} &= 0, \\
    \left(\sum_{u\in V_i}\bm{h}_u^{-1}\right)^{-1} &= \hat{\bm{\beta}}_{+} \cdot \left(\sum_{u\in V_i}\bm{h}_u^{-1}\right)^{-1} + \hat{\bm{\beta}}'_{+} \cdot \left(\sum_{u\in V_i}\bm{h}_u^{-1}\right)^{-1} \quad \forall i \in [n], 
\end{align*}

where $\hat{\bm{\beta}}_{+}$, $\hat{\bm{\beta}}'_{+}$ are the combined weights of $\bm{w}$'s in the positive direction, $\hat{\bm{\beta}}_{-}$, $\hat{\bm{\beta}}'_{-}$ are the combined weights of $\bm{w}$'s in the negative direction, and $c$ is a constant.

The above conditions can be satisfied with proper $\lambda_i$'s, so the model can fit all training data. Moreover, since the solution $\phi(G)^\top\bm{\beta}_{\text{NTK}}=\left(\sum_{u\in V}\bm{h}_u^{-1}\right)^{-1}$ is equivalent to the functional form of the target function \eqref{eq:harmonic}, GNN defined as in \eqref{eq:gnn_gnp_structure} can learn the \texttt{harmonic} task.
\end{proof}

Below, we prove GNNs with sum-aggregation and max-readout trained with squared loss in the NTK regime cannot extrapolate well on the \texttt{harmonic} task.

\begin{theorem}
\label{theorem:cannot:max}
Assume all nodes have the same scalar feature $1$. Then, one-layer GNNs with sum-aggregation and max-readout trained with squared loss in the NTK regime do not extrapolate well in the \texttt{harmonic} task.
\end{theorem}

\begin{proof}
The target function of \texttt{harmonic} task is
\[
f^\star(G) = \bigg( \sum_{v\in V} \bigg( \sum_{u\in N(v)} 1 \bigg)^{-1}\bigg)^{-1},
\]
so in order for one-layer GNNs with sum-aggregation and max-readout of the form
\[
\mathrm{MLP}\left( \max \sum_{v \in V} h_v \right), \quad h_v \text{ is the hidden vector for the node }v,
\]
to match the target function, $\mathrm{MLP}$ must learn some non-linear transform between $\max$ and the inverse function.
However, as shown in \citet{xu2021neural}, $\mathrm{MLP}$ converges to a linear function along directions from the origin. Hence, there always exist domains for which the GNN cannot learn the target function. 
\end{proof}

Similarly, we can show that one-layer GNNs with sum-aggregation and min/sum/mean-readout cannot learn the target function for some domain, meaning that they cannot extrapolate.

\section{Training Details}
\label{sec:app:details}
We used the open-source implementations of Set Transformer provided by the authors.
We used the open-source implementation of SAGPool in Pytorch Geometric \citep{Fey/Lenssen/2019} provided by the authors with the reported hyperparameter settings.
For all other models, we used the open-source implementations provided by the DGL framework \citep{wang2019dgl}. 

For all models, we used the mean squared loss (MSE) as training and validation loss functions, unless otherwise stated.
We performed a grid search to find the combination of hyperparameters that minimize the validation loss.
In all experiments, we used the RMSprop optimizer \citep{tieleman2012lecture} to train all models with \method, and for all baseline models, we additionally considered the Adam optimizer \citep{kingma2014adam} with default parameters (i.e.,  $\beta=(0.9, 0.999)$) and $\beta=(0.5, 0.999)$.

\subsection{Extrapolation on Graph-level Tasks (Related to Section \ref{sec:exp:graph})}
\label{sec:app:details:graph}

For each task, we generated Erdős–Rényi \citep{erdos1960evolution} random graphs with probabilities ranging from $0.1$ to $0.9$. We trained and validated our model using such graphs with at least 20 and at most 30 nodes, and we tested on such graphs with at least 50 and at most 100 nodes, following the procedure in \cite{xu2021neural}.
We generated $5,000$ graphs for training, $1,000$ graphs for validation, and $2,500$ graphs for test. For all nodes, we used the scalar $1$ as the node feature.

For further experiments with different structures, we generated $2,500$ graphs of each type among \texttt{ladder} graphs, \texttt{4-regular} random graphs,\footnote{The degree of every node is $4$.} random \texttt{tree}s,  \texttt{expanders},\footnote{We created Erdos-Renyi random graphs with probability $0.8$, following the procedure in \cite{xu2021neural}.} and Barabási–Albert (\texttt{BA}) \citep{barabasi1999emergence} random graphs\footnote{The number of edges to attach from a new node to existing nodes ranged from $0.05\times|V|$ to $0.4\times|V|$.}. They all have at least $50$ and at most $100$ nodes.

Table \ref{tab:graph} describes the hyperparameter search space for all graph-level tasks.

\begin{table}[h]
    \centering
    \caption{Search space for \texttt{maxdegree}, \texttt{harmonic}, and \texttt{invsize} tasks}
    \begin{subtable}{\linewidth}
        \centering
        \scalebox{0.85}{
        \begin{tabular}{cc}
        \toprule
        Hyperparameter & Selection pool \\
        \midrule
        Optimizer & RMSprop \\
        Learning rate for $p$ & 3e-2, 1e-2, 3e-3 \\
        Learning rate for the other parameters & 3e-2, 1e-2, 3e-3, 1e-3\\
        Norm clipping & 1e2, 1e4 \\
        \bottomrule
        \end{tabular}
        }
        \label{tab:graph:gnp}
        \caption{Search space for GIN with \method}
    \end{subtable}
    \begin{subtable}{\linewidth}
        \centering
        \scalebox{0.85}{
        \begin{tabular}{cc}
        \toprule
        Hyperparameter & Selection pool \\
        \midrule
        Optimizer & Adam, Adam with $\beta=(0.5, 0.999)$, RMSprop \\
        Learning rate & 3e-2, 1e-2, 3e-3, 1e-3\\
        Norm clipping & 1e2, 1e4 \\
        Number of iterations (for Set2Set) & $1$, $2$ \\
        \bottomrule
        \end{tabular}
        }
        \label{tab:graph:basic}
        \caption{Search space for GIN with baseline aggregation \& readout functions}
    \end{subtable}
    \begin{subtable}{\linewidth}
        \centering
        \scalebox{0.85}{
        \begin{tabular}{cc}
        \toprule
        Hyperparameter & Selection pool \\
        \midrule
        Optimizer & Adam, Adam with $\beta=(0.5, 0.999)$, RMSprop \\
        Learning rate & 1e-2, 3e-3, 1e-3, 3e-4\\
        Norm clipping & 1e2, 1e4 \\
        \bottomrule
        \end{tabular}
        }
        \label{tab:graph:sagpool}
        \caption{Search space for SAGPool}
    \end{subtable}
    \label{tab:graph}
\end{table}

\subsection{Extrapolation on Node-level Tasks (Related to Section \ref{sec:exp:node})}
\label{sec:app:details:node}

We created $5,000$ graphs for training, $1,000$ graphs for validation, and $2,500$ graphs of each type for test in the way described in Section~\ref{sec:exp:graph}.
Graphs for training and validation have at least $20$ and at most $40$ nodes, while those for test are larger with at least $50$ and at most $70$ nodes.
Target nodes is sampled uniformly at random among all nodes in each graph.

For \texttt{shortest}, we used the scalar $0$ as the feature of the target node, and used the scalar $10 \times |V|$ as the feature of the other nodes. The weight of each edge is drawn uniformly at random from $U(0, 5)$ in training and validation graphs, and from $U(0, 10)$ in test graphs. 
For \texttt{bfs}, we used the scalar $1$ as the feature of the target node and used the scalar $0$ as the feature for the other nodes.
For both tasks, we added self-loop with edge weight $0$ to every node.

Tables \ref{tab:node} describes the hyperparameter search space for all node-level tasks.

\begin{table}[h]
    \centering
    \caption{Search space for \texttt{bfs}, \texttt{shortest} tasks}
    \begin{subtable}{\linewidth}
        \centering
        \scalebox{0.85}{
        \begin{tabular}{cc}
        \toprule
        Hyperparameter & Selection pool \\
        \midrule
        Optimizer & RMSprop \\
        Learning rate for $p$ & 3e-2, 1e-2, 3e-3 \\
        Learning rate for the other parameters & 1e-2, 3e-3, 1e-3\\
        Norm clipping & 1e2, 1e4 \\
        \bottomrule
        \end{tabular}
        }
        \label{tab:node:gnp}
        \caption{Search space for GIN with \method}
    \end{subtable}
    \begin{subtable}{\linewidth}
        \centering
        \scalebox{0.85}{
        \begin{tabular}{cc}
        \toprule
        Hyperparameter & Selection pool \\
        \midrule
        Optimizer & Adam, Adam with $\beta=(0.5, 0.999)$, RMSprop \\
        Learning rate & 3e-2, 1e-2, 3e-3, 1e-3\\
        Norm clipping & 1e2, 1e4 \\
        \bottomrule
        \end{tabular}
        }
        \label{tab:node:basic}
        \caption{Search space for GIN with baseline aggregation functions}
    \end{subtable}
    \label{tab:node}
\end{table}

\subsection{Extrapolation on Set-related Tasks (Related to Section \ref{sec:exp:set})}
\label{sec:app:details:set}

For each task, we generated $4,000$ sets for training, $500$ sets for validation, and $500$ sets for test. For each set, the number of elements is sampled uniformly at random from $[20, 40)$ for training and validation sets, and from $[50, 100)$ for test sets. For $\mu_{\text{post}}$, $\sigma^2_{\text{post}}$, and $\hat{\mu}_{\text{MAP}}$, we sampled elements from $\mathcal{N}(\mu, 1^2)$ where $\mu \sim \mathcal{N}(0, 1^2)$. For $\hat{\sigma}^2_{\text{MAP}}$, we sampled elements $\mathcal{N}(5, \sigma^{2})$ where $\sigma \sim \text{InvGamma}(1, 15)$. As loss functions, we used MSE for $\mu_{\text{post}}$ and $\sigma^2_{\text{post}}$ and used the negative logarithm of the product\footnote{This product is proportional to the posterior probability} of the likelihood and the prior for $\hat{\mu}_{\text{MAP}}$ and $\hat{\sigma}^2_{\text{MAP}}$.

Tables \ref{tab:set} describes the hyperparameter search space for all set-related tasks.

\begin{table}[h]
    \centering
    \caption{Search space for $\mu_{\text{post}}$, $\sigma^2_{\text{post}}$, $\hat{\mu}_{\text{MAP}}$, and $\hat{\sigma}^2_{\text{MAP}}$}
    \begin{subtable}{\linewidth}
        \centering
        \scalebox{0.85}{
        \begin{tabular}{cc}
        \toprule
        Hyperparameter & Selection pool \\
        \midrule
        Optimizer & RMSprop \\
        Learning rate for $p$ & 3e-2, 1e-2, 3e-3\\
        Learning rate for the other parameters & 3e-2, 1e-2, 3e-3\\
        Norm clipping & 1e4 \\
        \bottomrule
        \end{tabular}
        }
        \label{tab:set:gnp}
        \caption{Search space for \method}
    \end{subtable}
    \begin{subtable}{\linewidth}
        \centering
        \scalebox{0.85}{
        \begin{tabular}{cc}
        \toprule
        Hyperparameter & Selection pool \\
        \midrule
        Optimizer & Adam, Adam with $\beta=(0.5, 0.999)$, RMSprop \\
        Learning rate & 3e-2, 1e-2, 3e-3, 1e-3\\
        Norm clipping & 1e4 \\
        \bottomrule
        \end{tabular}
        }
        \label{tab:set:basic}
        \caption{Search space for basic operators}
    \end{subtable}
    \begin{subtable}{\linewidth}
        \centering
        \scalebox{0.85}{
        \begin{tabular}{cc}
        \toprule
        Hyperparameter & Selection pool \\
        \midrule
        Optimizer & Adam, Adam with $\beta=(0.5, 0.999)$, RMSprop \\
        Learning rate & 1e-2, 1e-3, 1e-4\\
        Norm clipping & 1e4 \\
        Number of iterations (for Set2Set) & $1$, $2$ \\ 
        Encoder design (for Set transformer) & $2$ SAB blocks, $2$ ISAB blocks \\ 
        \bottomrule
        \end{tabular}
        }
        \label{tab:set:others}
        \caption{Search space for Set2Set and Set Transformer}
    \end{subtable}
    \label{tab:set}
\end{table}

\subsection{Graph Classification (Related to  Section \ref{sec:exp:graphcl})}

For original SAGPool \citep{lee2019self}, we used the optimal hyperparameter settings shared by the authors\footnote{\url{https://docs.google.com/spreadsheets/d/1JXGNOCQkRHDCQqNarteYpEuWnkNzNq_WFiQrIY276i0/edit?usp=sharing}}. For SAGPool equipped with \method, we used gradient clipping with a maximum gradient norm of $1000$ for the parameters of \method, and we used a different learning rate for $p$ of \method.
For \texttt{DD} , \texttt{PROTEINS}, and \texttt{NCI1}, we used $1\times$, $20\times$, and $10\times$ larger learning rates for $p$ than the original learning rates, respectively.

For original ASAPool \citep{ranjan2020asap}, we used the optimal hyperparameter settings shared by the authors\footnote{\url{https://github.com/malllabiisc/ASAP}}. For ASAPool equipped with \method, we used a different learning rate for $p$ and $q$ of \method.
For \texttt{DD} and \texttt{NCI1}, we used 3e-2 and 3e-3 for the learning rate for $p$ and $q$, respectively. For \texttt{PROTEINS}, we used 1e-1 for the learning rate for $p$, and 1e-2 for the learning rate for $q$.

\subsection{Influence Maximization (Related to Section \ref{sec:exp:im})}
For original MONSTOR \citep{ko2020monstor}, we used the optimal hyperparameter settings provided in the paper. For the parameters of \method, we used the RMSprop optimizer, and the learning rates were set to 3e-2 for $p$ and 3e-3 for $q$.

\section{Additional Experiments and Results}

\begin{table}[]
    \centering
    \caption{Statistics of real-world datasets}
    \scalebox{0.85}{
    \begin{tabular}{cccc}
         \toprule
         Dataset & Number of graphs & Average number of nodes & Average number of edges \\
         \midrule
         \texttt{D\&D} & 1178 & 284.3 & 715.7 \\
         \texttt{PROTEINS} & 1113 & 39.06 & 72.82 \\
         \texttt{NCI1} & 4110 & 29.87 & 32.30 \\
         \bottomrule
    \end{tabular}
    }
    \label{tab:real_stat}
\end{table}

\begin{table}[t]
    \centering
    \caption{Extrapolation performances of \method on real-world datasets. Only except for the maxdegree task on the \texttt{NCI1} dataset, \method showed near-perfect performance.}
    \scalebox{0.85}{
    \begin{tabular}{cccc}
        \toprule
        Task & \texttt{D\&D} & \texttt{PROTEINS} & \texttt{NCI1} \\
        \midrule
        \texttt{invsize} & 1.7$\pm$0.6 & 0.5$\pm$0.2 & 0.3$\pm$0.1 \\
        \texttt{harmonic} & 3.4$\pm$1.1 & 2.3$\pm$0.3 & 2.4$\pm$0.7 \\
        \texttt{maxdegree} & 3.4$\pm$1.3 & 2.8$\pm$1.1 & 22.4$\pm$12.4 \\
        \bottomrule
    \end{tabular}
    }
    \label{tab:real_extrapolate}
\end{table}

\begin{table}[t]
    \centering
    \caption{Extrapolation performance of baseline approaches on real-world datasets. Except for the \texttt{maxdegree} task, there was no combination of simple pooling functions that extrapolated well.}
    \scalebox{0.8}{
    \begin{tabular}{cccc}
        \toprule
        Task & \texttt{D\&D} & \texttt{PROTEINS} & \texttt{NCI1} \\
        \midrule
        \texttt{invsize} (best combination) & 100.0$\pm$0.0 (SortPooling) & 93.6$\pm$2.8 (SAGPool) & 37.9$\pm$0.0 (set2set) \\
        \texttt{harmonic} (best combination) & 552.7$\pm$1012.4 (sum, mean) & 110.2$\pm$3.6 (sum, max) & 39.5$\pm$1.5 (sum, max) \\
        \texttt{maxdegree} (ideal combination) & 0.0$\pm$0.0 (sum, max) & 0.0$\pm$0.0 (sum, max) & 0.0$\pm$0.0 (sum, max) \\
        \texttt{maxdegree} (2nd best combination) & 10.5$\pm$0.9 (sum, mean) & 30.8$\pm$3.9 (sum, mean) & 63.9$\pm$15.0 (sum, mean) \\
        \bottomrule
    \end{tabular}
    }
    \label{tab:real_baselines}
\end{table}

\subsection{Graph-level Extrapolation on Real-world Datasets \\ (Related to  Section \ref{sec:exp:graph})}
\label{sec:app:exp:realgraph}

We further tested the extrapolation performances of \method and baseline approaches using real-world graphs. For real-world graphs, we used \texttt{D\&D}, \texttt{PROTEINS}, and \texttt{NCI1}, which were also used for graph classification tasks in the paper. Table \ref{tab:real_stat} describes statistics of datasets. For evaluation, we ignored graphs with nodes with zero in-degrees.

In this experiment, we used a model trained using the Erdos–Rényi graphs described in Section \ref{sec:exp:graph}. As seen in the Table \ref{tab:real_extrapolate}, \method showed near-perfect extrapolation performance only except for the \texttt{maxdegree} task on the \texttt{NCI1} dataset. Even though the average number of nodes in the \texttt{D\&D} dataset is approximately $10$ times larger than that of the training dataset, the models trained with \method performed well. One of the possible reasons for the relatively high MAPE on the \texttt{NCI1} dataset is its extremely low average degree of nodes, which is roughly $2.17$. Note that the training dataset contains Erdos–Rényi random graphs with edge probabilities ranging from $0.1$ to $0.9$.

We also measured the test error of the baseline approaches, and we reported the test MAPE of the best-performing one in Table \ref{tab:real_baselines}. Except for the \texttt{maxdegree} task, there was no combination of simple pooling functions that extrapolated well. These results are consistent with the experiment results in the paper. On the \texttt{maxdegree} task, the second best combination (among the 16 combinations of \texttt{sum}, \texttt{max}, \texttt{mean}, and \texttt{min}) showed significantly worse extrapolation performance than the GIN model equipped with \texttt{GNP}.

\begin{table}[]
    \centering
    \caption{Test error on heterogeneous structures. Each row denotes the test MAPEs of the model trained using the same graph.}
    \begin{subtable}{\linewidth}
    \centering
    \scalebox{0.8}{
    \begin{tabular}{ccccccc}
         \toprule
         & \texttt{ER} & \texttt{BA} & \texttt{4regular} & \texttt{Expander} & \texttt{Tree} & \texttt{Ladder} \\
         \midrule
         \texttt{ER} & 1.2$\pm$0.3 & 0.9$\pm$0.3 & 0.8$\pm$0.3 & 1.9$\pm$1.0 & 0.8$\pm$0.3 & 0.8$\pm$0.3 \\
         \texttt{BA} & 1.4$\pm$1.3 & 1.1$\pm$0.7 & 1.1$\pm$0.5 & 1.9$\pm$2.7 & 1.0$\pm$0.5 & 1.1$\pm$0.5 \\
         \texttt{4regular} & 9.5$\pm$13.7 & 6.0$\pm$8.6 & 0.6$\pm$0.4 & 15.8$\pm$22.9 & 0.8$\pm$0.2 & 0.7$\pm$0.3 \\
         \texttt{Expander} & 1.2$\pm$0.4 & 1.3$\pm$0.4 & 1.9$\pm$1.8 & 1.0$\pm$0.5 & 5.8$\pm$7.3 & 2.8$\pm$2.8 \\
         \texttt{Tree} & 6.1$\pm$5.9 & 3.7$\pm$3.7 & 0.9$\pm$0.4 & 11.2$\pm$10.5 & 0.9$\pm$0.4 & 0.9$\pm$0.4 \\
         \texttt{Ladder} & 5.1$\pm$2.7 & 3.1$\pm$1.6 & 1.4$\pm$0.5 & 7.0$\pm$5.1 & 2.7$\pm$3.0 & 1.1$\pm$0.6 \\
         \bottomrule
    \end{tabular}
    }
    \caption{\texttt{invsize}}
    \end{subtable}
    \begin{subtable}{\linewidth}
    \centering
    \scalebox{0.8}{
    \begin{tabular}{ccccccc}
         \toprule
         & \texttt{ER} & \texttt{BA} & \texttt{4regular} & \texttt{Expander} & \texttt{Tree} & \texttt{Ladder} \\
         \midrule
         \texttt{ER} & 1.1$\pm$0.8 & 2.5$\pm$0.9 & 1.9$\pm$1.3 & 0.9$\pm$0.5 & 14.7$\pm$6.3 & 2.4$\pm$2.4 \\
         \texttt{BA} & 8.0$\pm$2.2 & 2.9$\pm$0.5 & 2.7$\pm$0.6 & 13.7$\pm$4.0 & 7.2$\pm$3.1 & 2.8$\pm$2.0 \\
         \texttt{4regular} & 80.4$\pm$0.8 & 62.1$\pm$0.8 & 1.5$\pm$1.6 & 92.5$\pm$0.7 & 165.3$\pm$19.2 & 39.2$\pm$5.2 \\
         \texttt{Expander} & 162.3$\pm$62.2 & 382.4$\pm$152.9 & 1202.6$\pm$526.6 & 7.8$\pm$2.2 & 3218.2$\pm$1491.6 & 1658.4$\pm$741.2 \\
         \texttt{Tree} & 119.5$\pm$57.8 & 116.4$\pm$66.7 & 66.7$\pm$22.5 & 126.0$\pm$61.7 & 4.1$\pm$5.0 & 45.0$\pm$15.1 \\
         \texttt{Ladder} & 85.8$\pm$0.4 & 72.4$\pm$0.6 & 24.8$\pm$0.9 & 94.8$\pm$0.2 & 99.8$\pm$7.3 & 4.0$\pm$0.9 \\
         \bottomrule
    \end{tabular}
    }
    \caption{\texttt{harmonic}}
    \end{subtable}
    \begin{subtable}{\linewidth}
    \centering
    \scalebox{0.8}{
    \begin{tabular}{ccccccc}
         \toprule
         & \texttt{ER} & \texttt{BA} & \texttt{4regular} & \texttt{Expander} & \texttt{Tree} & \texttt{Ladder} \\
         \midrule
         \texttt{ER} & 2.5$\pm$0.4 & 2.1$\pm$1.1 & 3.4$\pm$3.7 & 2.3$\pm$1.1 & 1.9$\pm$0.6 & 30.7$\pm$16.9 \\
         \texttt{BA} & 3.9$\pm$2.0 & 2.0$\pm$0.8 & 4.7$\pm$5.6 & 3.6$\pm$1.4 & 2.1$\pm$1.6 & 15.9$\pm$20.1 \\
         \texttt{4regular} & 86.7$\pm$3.5 & 89.2$\pm$2.1 & 5.0$\pm$8.1 & 92.3$\pm$3.3 & 15.5$\pm$1.5 & 38.8$\pm$8.5 \\
         \texttt{Expander} & 13.2$\pm$9.2 & 23.5$\pm$4.6 & 238.6$\pm$138.7 & 7.4$\pm$2.6 & 131.6$\pm$96.7 & 306.7$\pm$174.7 \\
         \texttt{Tree} & 78.8$\pm$28.2 & 57.7$\pm$21.3 & 41.3$\pm$13.2 & 126.9$\pm$44.7 & 2.6$\pm$0.5 & 31.0$\pm$22.2 \\
         \texttt{Ladder} & 91.2$\pm$0.2 & 92.6$\pm$0.1 & 24.1$\pm$1.3 & 95.4$\pm$0.0 & 35.3$\pm$1.1 & 2.1$\pm$1.5 \\
         \bottomrule
    \end{tabular}
    }
    \caption{\texttt{maxdegree}}
    \end{subtable}
    \label{tab:heterogeneous_all}
\end{table}

\subsection{Graph-level Extrapolation on Graphs with Different Structure Types (Related to  Section \ref{sec:exp:graph})}
\label{sec:app:exp:structures}

We trained a GNN using graphs of one structure type at a time then measured extrapolation error on the other structure types. In Table \ref{tab:heterogeneous_all}, each row denotes the test MAPEs of the model trained using the same graph. While the model trained using \texttt{ER} graphs or \texttt{BA} graphs extrapolated well on all three tasks, the model trained using the other graphs showed poor extrapolation performance. According to \citet{xu2021neural}, the distribution of training graphs can affect the extrapolation performance, and this can be one of the possible reasons why the model trained using \texttt{4regular}, \texttt{expander}, \texttt{tree}, \texttt{ladder} graphs showed poor extrapolation performance.

\subsection{Graph-level Extrapolation on Graphs with Different Node Feature Distributions (Related to Section \ref{sec:exp:graph})}
\label{sec:app:exp:nodefeats}

In the paper, we investigated the extrapolation performances in graph-level and node-level tasks on graphs with different sizes and structures. We also performed experiments on graphs with different edge feature distributions for the \texttt{shortest} task.

We additionally performed graph-level experiments for testing extrapolation to out-of-distribution node features. As in \citet{xu2021neural}, 3-dimensional node features drawn from $U(0, 5)$ were used in training and validation data, and those drawn from $U(0, 10)$ were used in test data.

\begin{table}[t]
    \centering
    \caption{Test error on graph-level tasks with different node feature distributions.}
    \scalebox{0.85}{
    \begin{tabular}{cccc}
        \toprule
        Task & \texttt{invsize} & \texttt{harmonic} &  \texttt{maxdegree} \\
        \midrule
        Test MAPE & 0.8$\pm$0.6 & 2.4$\pm$1.7 & 4.7$\pm$1.4 \\
        \bottomrule
    \end{tabular}
    }
    \label{tab:nodefeat}
\end{table}

\begin{table}[t]
    \centering
    \caption{Test error on graph-level tasks with different activation functions.}
    \scalebox{0.85}{
    \begin{tabular}{cccc}
        \toprule
        Task & \texttt{invsize} & \texttt{harmonic} &  \texttt{maxdegree} \\
        \midrule
        ReLU & 0.7$\pm$0.5 & 5.1$\pm$0.9 & 7.3$\pm$2.1 \\
        LeakyReLU & 0.3$\pm$0.2 & 4.6$\pm$1.4 & 5.4$\pm$1.6 \\
        ELU & 0.2$\pm$0.2 & 5.8$\pm$0.7 & 6.2$\pm$3.1 \\
        \bottomrule
    \end{tabular}
    }
    \label{tab:act}
\end{table}

We reported the test error in Table \ref{tab:nodefeat}. As shown in the table, the error was slightly larger than that in the original settings without node features. However, the error was still reasonably low, and \method outperformed baseline approaches especially on the \texttt{invsize} and \texttt{harmonic} tasks.

\subsection{Graph-level Extrapolation with Various Activation Functions \\ (Related to Section \ref{sec:exp:graph})}
\label{sec:app:exp:activations}

We performed an additional graph-level experiment with a variant of GNP for handling negative inputs and a wider range of activation functions. Since the original \method can only take non-negative inputs, we replaced ReLU to the absolute function for processing the inputs and then used an activation function. We considered ReLU, ELU, and LeakyReLU as the activation function. We compared the extrapolation error in each setting in Table \ref{tab:act}, and GNP with the aforementioned changes showed performance comparable to original GNP.

\subsection{Test MAPE on Graph-level Tasks (Related to Section \ref{sec:exp:graph})}
\label{sec:app:exp:allresults}

In Table \ref{tab:exp:graph}, we reported test MAPEs and standard deviations for all $19$ competitors and \method on the graph-level tasks.

\begin{table}[h]
  \caption{Extrapolation performances on three graph-level tasks. %
  We reported test MAPEs and standard deviations.
  Near-perfect scores are in bold, and scores significantly better than those in completely failed cases are underlined.
  }
  \centering
  \begin{subtable}{\linewidth}
  \centering
  \scalebox{0.85}{
  \begin{tabular}{c|cccc||c|c}
    \toprule
    \multirow{2}{*}{Readout} & \multicolumn{4}{c||}{Aggregation} & \multirow{2}{*}{Readout} & \multirow{2}{*}{Test MAPE} \\
    \cmidrule(r){2-5}
    & sum & max & mean & min & & \\
    \midrule
    sum  & 376.1$\pm$378.0 & 257.1$\pm$351.3 & 257.1$\pm$351.3 & 257.1$\pm$351.3 & SortPool  & 100.0$\pm$0.0 \\ 
    max  & 101.0$\pm$7.6 & 179.1$\pm$44.2 & 179.1$\pm$44.2 & 179.1$\pm$44.2 &  Set2Set  & 198.8$\pm$0.5 \\
    mean & 116.9$\pm$6.6 & 179.9$\pm$44.7 & 179.1$\pm$44.2 & 179.9$\pm$44.7 & SAGPool & 178.7$\pm$10.4\\
    min  & 139.6$\pm$54.2 & 179.1$\pm$44.2 & 179.1$\pm$44.2 & 179.1$\pm$44.2 & \textbf{\method} & \textbf{1.2$\pm$0.3}\\
    \bottomrule
  \end{tabular}
  }
  \caption{\texttt{invsize}}
  \end{subtable}
  \begin{subtable}{\linewidth}
  \centering
  \scalebox{0.85}{
  \begin{tabular}{c|cccc||c|c}
    \toprule
    \multirow{2}{*}{Readout} & \multicolumn{4}{c||}{Aggregation} & \multirow{2}{*}{Readout} & \multirow{2}{*}{Test MAPE} \\
    \cmidrule(r){2-5}
    & sum & max & mean & min & & \\
    \midrule
    sum  & { }109.6$\pm$19.0{ } & { }121.1$\pm$28.9{ } & { }151.5$\pm$73.1{ } & { }121.1$\pm$28.9{ } & SortPool  & { }76.8$\pm$13.0{ } \\
    max  & 73.0$\pm$2.3 & 76.0$\pm$3.6 & 76.4$\pm$3.5 & 76.0$\pm$3.6 &  Set2Set  & 78.2$\pm$4.0 \\
    mean & 95.7$\pm$9.5 & 75.9$\pm$3.6 & 76.3$\pm$3.5 & 75.9$\pm$3.6 & SAGPool & 26.9$\pm$21.0 \\
    min  & 91.1$\pm$12.2 & 76.0$\pm$3.6 & 76.3$\pm$3.5 & 76.0$\pm$3.6 &  \textbf{\method} & \textbf{1.1$\pm$0.8}  \\
    \bottomrule
   \end{tabular}
  }
  \caption{\texttt{harmonic}}
  \end{subtable}
  \begin{subtable}{\linewidth}
  \centering
  \scalebox{0.85}{
  \begin{tabular}{c|cccc||c|c}
    \toprule
    \multirow{2}{*}{Readout} & \multicolumn{4}{c||}{Aggregation} & \multirow{2}{*}{Readout} & \multirow{2}{*}{Test MAPE} \\
    \cmidrule(r){2-5}
    & sum & max & mean & min & & \\
    \midrule
    sum  & { }{ }60.5$\pm$22.1{ }{ } & { }{ }{ }50.5$\pm$2.1{ }{ }{ } & { }{ }{ }49.9$\pm$0.5{ }{ }{ } & { }{ }{ }50.5$\pm$2.1{ }{ }{ } & SortPool  & { }\textbf{0.0$\pm$0.0}{ } \\
    max  & \textbf{0.0$\pm$0.0} & 59.7$\pm$0.3 & 59.7$\pm$0.3 & 59.7$\pm$0.3 & Set2Set  & \textbf{0.0$\pm$0.0} \\
    mean & \underline{16.3$\pm$2.4} & 59.7$\pm$0.3 & 59.7$\pm$0.3 & 59.7$\pm$0.3 & SAGPool  & { }{ }51.4$\pm$1.8{ }{ } \\
    min  & 25.5$\pm$3.2 & 59.7$\pm$0.3 & 59.7$\pm$0.3 & 59.7$\pm$0.3 & \textbf{\method} & \textbf{2.5$\pm$0.4}  \\
    \bottomrule
  \end{tabular}
   }
  \caption{\texttt{maxdegree}}
  \end{subtable}
  \label{tab:exp:graph}
\end{table}

\begin{table}[]
    \centering
    \caption{Test error with different masking schemes}
    \scalebox{0.85}{
    \begin{tabular}{cccc}
        \toprule
        Task & without masking & masking $\method^{+}$ & masking $\method^{-}$ \\
        \midrule
        \texttt{invsize} & \textbf{1.2$\pm$0.3} & \textbf{1.1$\pm$0.1} & 99.6$\pm$1.0 \\
        \texttt{harmonic} & \textbf{1.1$\pm$0.8} & \textbf{1.0$\pm$0.7} & 100.1$\pm$0.1 \\
        \texttt{maxdegree} & \textbf{2.5$\pm$0.4} & 100.0$\pm$0.2 & \textbf{2.5$\pm$0.4} \\
        \bottomrule
    \end{tabular}
    }
    \label{tab:behavior_app}
\end{table}

\subsection{Behaviors of $\method^{+}$ and $\method^{-}$ for Graph-level Tasks \\ (Related to Section \ref{sec:exp:behavior})}
\label{sec:app:exp:behaviors}

We analyzed the behavior of the negative \method on three graph-level tasks that we performed in the paper. In all experiments, we found that either $\method^{+}$ or $\method^{-}$ tends to dominate the other side. To validate the observation, we masked the output of $\method^{+}$ and $\method^{-}$ for readout to 0 on the graph-level tasks.

As seen in Table \ref{tab:behavior_app}, masking the output of $\method^{-}$ on the maxdegree task and masking the output of $\method^{+}$  on the other tasks do not significantly affect the extrapolation performance. When we masked the opposite part of \method, however, the test MAPE was near 100. These results imply that the effect of the dominated part on the output of the model is negligible. That is, when the optimal pooling function is max, the negative \method has almost no effect on determining the output. Similarly, when the optimal function is \method with $(p, q) = (-1, 0)$, the positive \method has almost no effect on determining the output.

\section{Closed-form Solutions for Set-related Tasks \\ (Related to Section \ref{sec:exp:set})}

In Table \ref{tab:exp:set_gt}, we provided the closed-form solutions for each task. 
\begin{table}[h]
    \centering
    \caption{A closed-form solution for each task when the input set $S = \{x_1, x_2, \cdots, x_n\}$ is given.}
    \begin{tabular}{ll}
        \toprule
        Task & Closed form solution \\
        \midrule
        $\mu_{\text{post}}$ & $\left(\frac{1}{\sigma_0^2} + \frac{n}{\sigma^2}\right)^{-1} \cdot \left(\frac{\mu_0}{\sigma_0^2} + \frac{1}{\sigma^2}\sum_{i=1}^{n}x_i\right)$\\
        $\sigma^2_{\text{post}}$ & $\left(\frac{1}{\sigma_0^2} + \frac{n}{\sigma^2}\right)^{-1}$ \\
        $\hat{\mu}_{\text{MAP}}$ & $\left(\frac{1}{\sigma_0^2} + \frac{n}{\sigma^2}\right)^{-1} \cdot \left(\frac{\mu_0}{\sigma_0^2} + \frac{1}{\sigma^2}\sum_{i=1}^{n}x_i\right)$ \\
        $\hat{\sigma}^2_{\text{MAP}}$ & $\left(\alpha+\frac{n}{2}+1\right)^{-1} \cdot \left(\beta+\frac{1}{2}\sum_{i=1}^{n}(x_i-\mu)^2\right)$\\
        \bottomrule
    \end{tabular}
    \label{tab:exp:set_gt}
\end{table}

\section{Details about Influence Maximization and MONSTOR \\ (Related to Section \ref{sec:exp:graphcl})}
\label{sec:app:im}
Influence Maximization (IM) \citep{kempe2003maximizing} is one of the most extensively studied NP-hard problems on social networks due to its practical applications in viral marketing and computational epidemiology. The goal of the problem is to choose a given number of seed nodes (i.e., a set of initially activated nodes) that maximize the influence through a given graph under a diffusion model. In this experiment, we used the Independent Cascade (IC) model as the diffusion model. In the IC model, each link $(u, v)$ has an activation probability $p_{uv}$. When a node $u$ is newly activated and a neighbor $v$ is not activated yet, the node $u$ has exactly one chance to activate the node $v$ with the probability $p_{uv}$, and the diffusion process ends when every activated node fails to activate any new node. In the model, the influence is the number of activated nodes after the diffusion process ends.

MONSTOR estimates the influence given a graph and a seed set. To train the model, we generated a dataset consisting of pairs of an input graph and a set of randomly chosen seed nodes. To generate ground-truth answers, we ran $10,000$ Monte-Carlo simulations and recorded the probability $\pi_{u,i}$ that each node $u$ is activated until the $i$-th step. We first trained the base model $M$ to estimate $\pi_i$ given $\pi_{i-1}$, \ldots , $\pi_{i-d}$. 
MONSTOR is constructed by stacking $s$ times the base model $M$, and $s$ is chosen to minimize squared loss between the ground-truth influences and the estimated influences on the validation set. Since influence maximization is a submodular maximization problem, we used UBLF \citep{zhou2013ublf} or CELF \citep{leskovec2007cost} equipped with MONSTOR, which greedily selects seed nodes.

\section{Code \& Data}
All assets used in the paper, including the training/evaluation code and the trained models with \method, are contained in the supplemental material. %
All assets we used from DGL\footnote{\url{https://github.com/dmlc/dgl}} \citep{wang2019dgl} and Pytorch Geometric\footnote{\url{https://github.com/rusty1s/pytorch_geometric}} \citep{Fey/Lenssen/2019} are available under the Apache license $2.0$ and MIT license, respectively. The implementation of Set Transformer\footnote{\url{https://github.com/juho-lee/set_transformer}} \citep{lee2019set} that we used is available under the MIT License. The implementation of ASAPool\footnote{\url{https://github.com/malllabiisc/ASAP}} \citep{ranjan2020asap} that we used is available under the Apache license $2.0$. For the other assets, we were unable to find their licenses. For the SAGPool \citep{lee2019self} implementation in Pytorch Geometric, the dataset generators for the graph-level and node-level tasks \citep{xu2021neural}, and the MONSTOR \citep{ko2020monstor} implementation in DGL, we used the code on the GitHub repositories\footnote{\url{https://github.com/inyeoplee77/SAGPool}}${}^,$\footnote{ \url{https://github.com/jinglingli/nn-extrapolate}}${}^,$\footnote{ \url{https://github.com/jihoonko/asonam20-monstor}} shared by the authors of the original papers. 
We accessed TUDataset\footnote{\url{https://chrsmrrs.github.io/datasets/}} \citep{morris2020tu} using PyTorch Geometric.